\documentclass{article}

\usepackage{algorithm}
\usepackage{algpseudocode}

\usepackage[final]{neurips_2020}

\usepackage{appendix}
\usepackage[utf8]{inputenc} 
\usepackage[T1]{fontenc}    
\usepackage{hyperref}       
\usepackage{url}            
\usepackage{booktabs}       
\usepackage{amsfonts}       
\usepackage{nicefrac}       
\usepackage{microtype}      

\usepackage[table]{xcolor}
\usepackage{lipsum}
\usepackage{amssymb}
\usepackage{amsthm}
\usepackage{mathtools}
\usepackage{xfrac}
\usepackage{thmtools,thm-restate}

\DeclarePairedDelimiter\floor{\lfloor}{\rfloor}

\newcommand{\com}{\Theta}

\newcommand{\norm}[1]{\left\lVert#1\right\rVert}

\DeclareMathOperator*{\argmin}{arg\,min}
\newcommand{\proj}{\ensuremath{\text{\rm Proj}}}

\newtheorem{theorem}{Theorem}[section]

\newtheorem{lemma}[theorem]{Lemma}
\newtheorem{assumption}[theorem]{Assumption}
\newtheorem{remark}[theorem]{Remark}

 \makeatletter
 \renewenvironment{proof}[1][\proofname]{\par
   \pushQED{\qed}
   \normalfont \topsep6\p@\@plus6\p@\relax
   \trivlist
   \item[\hskip\labelsep
         \bfseries
     #1\@addpunct{.}]\ignorespaces
 }{
   \popQED\endtrivlist\@endpefalse
 }
 \makeatother
\def\eop{$\rule{1.3ex}{1.3ex}$}
\renewcommand\qedsymbol\eop

\usepackage[textwidth=2.0cm, textsize=tiny]{todonotes} 

\title{Exploiting Higher Order Smoothness in Derivative-free Optimization and Continuous Bandits}

\author{Arya Akhavan\\
  Istituto Italiano di Tecnologia\\
  and\\
  CREST, ENSAE, IP Paris\\
  \texttt{aria.akhavanfoomani@iit.it}
  \And
  Massimiliano Pontil\\
  Istituto Italiano di Tecnologia\\
  and\\
  University College London\\
  \texttt{massimiliano.pontil@iit.it }
   \AND
 Alexandre B. Tsybakov \\
   CREST, ENSAE, IP Paris \\
   \texttt{alexandre.tsybakov@ensae.fr } \\
}

\begin{document}

\maketitle

\begin{abstract}
We study the problem of zero-order optimization of a strongly convex function. The goal is to find the minimizer of the function by a sequential exploration of its 
values, under measurement noise. We study the impact of higher order smoothness properties of the function on the optimization error and on the cumulative regret. To solve this problem we consider a randomized approximation of the projected gradient descent algorithm. The gradient is estimated by a randomized procedure involving two function evaluations and a smoothing kernel. We derive upper bounds for this algorithm both in the constrained and unconstrained settings and prove minimax lower bounds for any sequential search method. Our results imply that the zero-order algorithm is nearly optimal in terms of sample complexity and the problem parameters. Based on this algorithm, we also propose an estimator of the minimum value of the function achieving almost sharp oracle behavior. We compare our results with the state-of-the-art, highlighting a number of key improvements.
\end{abstract}

\section{Introduction}
We study the problem of zero-order stochastic optimization, in which we aim to minimize 
an unknown strongly convex function via a sequential exploration of its function values, 
under measurement error, and a closely related problem of continuous (or continuum-armed) stochastic bandits.  These problems 
have received significant attention in the literature, see \cite{agarwal2010,agarwal2011,BP2016,bartlett2018simple,belloni2015escaping,bubeck2012regret,bubeck2017kernel,dvurechensky2018accelerated,hu2016bandit,liang2014zeroth,wang2017stochastic,flaxman2004,Locatelli,malherbe2017global,Shamir13,NS17,Shamir17,duchi2015,rakhlin2011making,saha2011,jamieson2012,Shalev-Shwartz2011}, and are fundamental for many applications in which the derivatives of the function are either too expensive or impossible to compute. 
%
A principal goal of this paper is to exploit higher order smoothness properties of the underlying function in order to improve the performance of search algorithms. We derive upper bounds on the estimation error for a class of projected gradient-like algorithms, as well as close matching lower bounds, that characterize the role played by the number of iterations,  
the strong convexity parameter, the smoothness parameter, the number of variables, and the noise level.

Let $f: \mathbb{R}^d \rightarrow \mathbb{R}$ be the function that we wish to minimize over a closed convex subset $\com$ of $\mathbb{R}^d$. 
Our approach, outlined in Algorithm \ref{algo}, builds upon previous work in which a sequential algorithm queries at each iteration a pair of function values, under a general noise model. 
Specifically, at iteration $t$ the current guess $x_t$ for the minimizer of $f$ is used to build two perturbations $x_t+\delta_t$ and $x_t-\delta_t$, where the function values are queried subject to additive measurement errors $\xi_t$ and $\xi_t'$, respectively. The values $\delta_t$ can be chosen in different ways. In this paper, we set $ \delta_t=h_tr_r\zeta_t$ (Line 1), where $h_t>0$ is a suitably chosen small parameter, $r_t$ is random and uniformly distributed on $[-1,1]$, and $\zeta_t$ is uniformly distributed on the unit sphere.  The estimate for the gradient is then computed at Line 2 and used inside a projected gradient method scheme to compute the next exploration point. 
We introduce a suitably chosen kernel $K$ that allows us to take advantage of higher order smoothness of~$f$. 

The idea of using randomized procedures for derivative-free stochastic optimization can be traced back to Nemirovski and Yudin \cite[Sec.~9.3]{NY1983} who suggested an algorithm with one query per step at point $x_t+h_t\zeta_t$, with $\zeta_t$ uniform on the unit sphere. Its versions with one, two or more queries were studied in several papers including~\cite{agarwal2010,BP2016,flaxman2004,Shamir17}. 
Using two queries per step leads to better performance bounds as emphasized in~\cite{PT90, agarwal2010,BP2016,flaxman2004,Shamir17,duchi2015}.  
Randomizing sequences other than uniform on the sphere were also explored: $\zeta_t$ uniformly distributed on a cube \cite{PT90}, Gaussian $\zeta_t$~\cite{Nesterov2011, NS17}, $\zeta_t$ uniformly distributed on the vertices of a cube~\cite{Shamir13} or satisfying some general assumptions~\cite{dippon,duchi2015}. Except for \cite{PT90,dippon,BP2016}, these works study settings with low smoothness of $f$ (2-smooth or less) and do not invoke kernels $K$ (i.e. $K(\cdot)\equiv 1$ and $r_t\equiv 1$ in Algorithm \ref{algo}). The use of randomization with smoothing kernels was proposed by Polyak and Tsybakov \cite{PT90} and further developed by Dippon \cite{dippon}, and
Bach and Perchet~\cite{BP2016} to whom the current form of Algorithm~\ref{algo} is due. 
%
\begin{algorithm}[t!]
\caption{Zero-Order Stochastic Projected Gradient} \label{algo}
\begin{algorithmic}
\State 
\State {\bfseries Requires} ~ Kernel $K :[-1,1]\rightarrow \mathbb{R}$, step size $\eta_t>0$ and parameter $h_t$, for $t=1,\dots,T$
\vspace{.1cm}
\State {\bfseries Initialization} ~ Generate scalars $r(1),\dots,r(T)$ uniformly on the interval $[-1,1]$, vectors $\zeta(1),\dots,\zeta(T)$ uniformly distributed on the unit sphere $S_d=\{\zeta\in \mathbb{R}^d:  \norm{\zeta}=1\}$, and choose $x(1)\in \com$
\vspace{.1cm}
\State {\bfseries For}  $t=1,\dots, T$
\vspace{.1cm}
\State \qquad {1.} ~~~Let $y(t) = f(x(t)+h_tr(t)\zeta(t))+ \xi(t)$ and $y'(t) = f(x(t)-h_tr(t)\zeta(t))+ \xi'(t),$
\vspace{.05cm}
\State \qquad 2. ~~~Define $\hat{g}(t) = \frac{d}{2h_t} (y(t)- y'(t)) \zeta(t) K(r(t))$
\vspace{.05cm}
\State \qquad 3. ~~~Update $x_{t+1}= x_{t}- \eta_t \hat{g}(t)$
\vspace{.05cm}
\State {\bfseries Return} ~ $(x(t))_{t=1}^T$
\end{algorithmic}
\end{algorithm}

In this paper we consider higher order smooth functions $f$ satisfying the generalized H\"older condition with parameter $\beta \geq 2$, cf. inequality~\eqref{eq:Hclass} below. For integer $\beta$, this parameter can be roughly interpreted as the number of bounded derivatives. Furthermore, we assume that $f$ is $\alpha$-strongly convex. For such functions, we address the following two main questions: 
\begin{itemize}
\item[(a)] What is the performance of Algorithm \ref{algo}  in terms of the cumulative regret and optimization error, namely what is the explicit dependency of the rate on the main parameters $d,T,\alpha,\beta$?
\item[(b)] What are the fundamental limits of any sequential search procedure expressed in terms of minimax
optimization error? 
\end{itemize}
To handle task (a), we prove upper bounds for Algorithm \ref{algo}, and to handle (b), we prove minimax lower bounds for any sequential search method. 

{\bf Contributions.} Our main contributions can be summarized as follows: {\bf i)} Under an adversarial noise assumption (cf. Assumption \ref{ass1} below), we establish for all $\beta\ge 2$ upper bounds of the order $\frac{d^2}{\alpha}T^{-\frac{\beta-1}{\beta}}$ for the optimization risk and $\frac{d^2}{\alpha}T^{\frac{1}{\beta}}$ for the cumulative regret of Algorithm \ref{algo}, both for its constrained and unconstrained versions; {\bf ii)} In the case of independent noise satisfying some natural assumptions (including the Gaussian noise), we prove a minimax lower bound of the order $\frac{d}{\alpha}T^{-\frac{\beta-1}{\beta}}$ for the optimization risk when $\alpha$ is not very small. This shows that to within the factor of $d$ the bound for Algorithm \ref{algo} cannot be improved for all $\beta\ge 2$; {\bf iii)} We show that, when $\alpha$ is too small, below some specified threshold, higher order smoothness does not help to improve the convergence rate. We prove that in this regime the rate cannot be faster than $d/\sqrt{T}$, which is not better (to within the dependency on $d$) than for derivative-free minimization of simply convex functions \cite{agarwal2011,liang2014zeroth}; {\bf iv)} For $\beta=2$, we obtain a bracketing of the optimal rate between $O(d/\sqrt{\alpha T})$ and $\Omega(d/(\max(1,\alpha)\sqrt{T}))$. In a special case when $\alpha$ is a fixed numerical constant, this validates a conjecture in \cite{Shamir13} (claimed there as proved fact) that the optimal rate for $\beta=2$ scales as $d/\sqrt{T}$; {\bf v)} We propose a simple algorithm of estimation of the value $\min_x f(x)$ requiring three queries per step and attaining the optimal rate $1/\sqrt{T}$ for all $\beta\ge 2$. The best previous work on this problem  \cite{belitser} suggested a method with exponential complexity and proved a bound of the order $c(d,\alpha)/\sqrt{T}$ for $\beta> 2$ where $c(d,\alpha)$ is an unspecified constant.  

{\bf Notation.} Throughout the paper we use the following notation. We let $\langle \cdot, \cdot \rangle$ and $\|\cdot\|$ be the standard inner product and Euclidean norm on $\mathbb{R}^d$, respectively. For every close convex set $\com\subset \mathbb{R}^d$ and $x\in \mathbb{R}^d$ we denote by $\proj_\com(x) = {\rm argmin} \{ \|z-x\| : z \in \com\}$ the Euclidean projection of $x$ to $\com$. We assume everywhere that $T\ge 2$. We denote by ${\cal F_\beta}(L)$ the class of functions with H\"older smoothness $\beta$ (inequality~\eqref{eq:Hclass} below). Recall that $f$ is $\alpha$-strongly convex for some $\alpha>0$ if, for any $x,y \in \mathbb{R}^d$ it holds that $f(y) \geq f(x) + \langle \nabla f(x),y-x\rangle + \frac{\alpha}{2} \|x-y\|^2$.
We further denote by ${\cal F_{\alpha,\beta}}(L)$ the class of all $\alpha$-strongly convex functions belonging to ${\cal F_\beta}(L)$. 

{\bf Organization.} 
We start in Section~\ref{sec:BM}~with some preliminary results on the gradient estimator.~Section~\ref{sec:UB}  presents our upper bounds for Algorithm \ref{algo}, both in the constrained
and unconstrained case.   
In Section \ref{sec:sec4} we observe that a slight modification of Algorithm 1 can be used to estimated the minimum value (rather than the minimizer) of $f$. Section  \ref{sec:sec4} presents improved upper bounds in the case $\beta=2$.
In Section~\ref{sec:LB} we establish minimax lower bounds.  
Finally, Section \ref{sec:PW} contrasts our results with previous work in the literature and discusses future directions of research.


\section{Preliminaries}
\label{sec:BM}
In this section, we give the definitions, assumptions and basic facts that will be used throughout the paper.~For $\beta>0$, let  $\ell$ be the greatest integer strictly less than $\beta$. We denote by ${\cal F_\beta}(L)$ the set of all functions $f:\mathbb{R}^d\to \mathbb{R}$ that are  $\ell$ times differentiable and satisfy, for all $x,z \in \com$ the H\"older-type condition
\begin{equation}
\bigg|f(z)-\sum_{0\le |m|\leq \ell} \frac{1}{m!}D^{m}f(x)(z-x)^{m} \bigg|\leq L \|z-x\|^{\beta},
\label{eq:Hclass}
\end{equation}
where $L >0$, the sum is over the multi-index $m=(m_{1},...,m_{d}) \in \mathbb{N}^d$, we used the notation $m!=m_{1}! \cdots m_{d}!$, $|m|=m_{1}+ \cdots+m_{d}$, and we defined 
\[
D^{m}f(x)\nu^{m} = \frac{\partial ^{|m|}f(x)}{\partial ^{m_{1}}x_{1} \cdots\partial ^{m_{d}}x_{d}}\nu_{1}^{m_{1}} \cdots \nu_{d}^{m_{d}}, \quad \forall \nu=(\nu_1,\dots, \nu_d) \in \mathbb{R}^{d}. 
\]
In this paper, we assume that the gradient estimator defined by Algorithm \ref{algo} uses a kernel function $K:[-1,1]\to \mathbb{R}$ satisfying 
\begin{equation}
\int K(u) du =0,  \int u K(u) du =1,  \int u^j K(u) du =0, \ j=2,\dots, \ell,~
\int |u|^{\beta} |K(u)| du  <\infty.
\end{equation}
Examples of such kernels obtained as weighted sums of Legendre polynomials are given in \cite{PT90} and further discussed in \cite{BP2016}. 

\begin{assumption} 
\label{ass1}
It holds, for all $t \in \{1,\dots,T\}$,   that: \text{(i)} the random variables $\xi_t$ and $\xi_t'$ are independent from $\zeta_t$ and from $r_t$, and the random variables $\zeta_t$ and $r_t$ are independent; \text{(ii)} $ \mathbb{E} [\xi_t^2]\le \sigma^2,$ and  $\mathbb{E} [(\xi_t')^2]\le \sigma^2$, where $\sigma\ge 0$.
\end{assumption}
Note that we do not assume $\xi_t$ and $\xi_t'$ to have zero mean. Moreover, they can be non-random and no independence between noises on different steps is required, so that the setting can be considered as adversarial. Having such a relaxed set of assumptions is possible because of randomization that, for example, allows the proofs go through without assuming the zero mean noise. 

We will also use the following assumption.

\begin{assumption}\label{ass:lip}
Function $f:\mathbb{R}^d\to \mathbb{R}$ is 2-smooth, that is, differentiable on $\mathbb{R}^d$ and such that $\|\nabla f(x)-\nabla f(x')\|\leq {\bar L} \|x-x'\|$ for all $x,x' \in \mathbb{R}^d$, where $\bar L>0$. 
\end{assumption}
It is easy to see that this assumption implies that 
$f\in {\cal F}_2(\bar L/2)$. 
The following lemma gives a bound on the bias of the gradient estimator. 
\begin{restatable}{lemma}{firstlemma} 
\label{lem:1}
Let $f \in {\cal F}_\beta(L)$, with $\beta > 1$ and let Assumption \ref{ass1} hold. Let ${\hat g}_t$ and $x_t$ be defined by Algorithm \ref{algo} and let $\kappa_\beta = \int |u|^{\beta} |K(u)| du.$ Then
\begin{equation}
\|\mathbb{E} [{\hat g}_t \hspace{.03truecm}|\hspace{.03truecm} x_t] - \nabla f(x_t)\| \leq  \kappa_\beta L  d h_t^{\beta-1}.
\label{eq:}
\end{equation}
\end{restatable}
{If $K$ be a weighted sum of Legendre polynomials, $\kappa_\beta \le 2\sqrt{2}\beta$, with $\beta \geq 1$ (see e.g.,  \cite[Appendix~A.3]{BP2016}).}

The next lemma provides a bound on the stochastic variability of the estimated gradient
by controlling its second moment.  
\begin{restatable}{lemma}{sndlemma} 
\label{lem:2unco}
Let Assumption \ref{ass1} hold, let ${\hat g}_t$ and $x_t$ be defined by Algorithm \ref{algo} and set $\kappa = \int K^2(u) du$. Then
\vspace{-2mm}
\begin{itemize}
    \item[(i)] If $\com \subseteq \mathbb{R}^d$, $\nabla f(x^*)=0$ and Assumption \ref{ass:lip} holds, 
    \[
\mathbb{E} [\|{\hat g}_t \|^2 \hspace{.04truecm}|\hspace{.04truecm} x_t] \leq 
9\kappa{\bar L}^2  \left(d\|x_t-x^*\|^2 + \frac{d^2h_t^2}{8} \right) + \frac{3\kappa d^2 \sigma^2}{2h_t^2}, 
\]
\vspace{-2mm}
    \item[(ii)] If $f\in \mathcal{F}_2(L)$ and $\com$ is a closed convex subset of $\mathbb{R}^d$ such that $\displaystyle{\max_{x \in \com}} \| \nabla f(x)\|\le G$,  then 
    \[
\mathbb{E} [\|{\hat g}_t \|^2 \hspace{.04truecm}|\hspace{.04truecm} x_t] \leq 
9\kappa  \left(G^2 d + \frac{{L}^2d^2h_t^2}{2} \right) + \frac{3\kappa d^2 \sigma^2}{2h_t^2}. 
\]
\end{itemize}
\end{restatable}
%
\section{Upper bounds}
\label{sec:UB}
In this section, we provide upper bounds on the cumulative regret and on the optimization error of Algorithm \ref{algo}, {which are defined as
$$\sum_{t=1}^T \mathbb{E} [ f(x_t) - f(x)],$$
and
$$\mathbb{E} [ f({\hat x}_T) - f(x^*)],$$
 respectively, where $x \in \com$ and ${\hat x}_T$ is an estimator after $T$ queries. Note that the provided upper bound for cumulative regret is valid for any $x \in \com$}. 

First we  consider Algorithm \ref{algo} when the convex set $\com$ is bounded (constrained case).  
\begin{restatable}{theorem}{thdtheorem}{\em (Upper Bound, Constrained Case.)}
\label{thm:1}
	Let $f \in {\cal F}_{\alpha, \beta}(L)$ with $\alpha,L>0$ and $\beta \geq 2$.
	Let Assumptions \ref{ass1} and \ref{ass:lip} hold and let $\com$ be a convex compact subset  of $\mathbb{R}^d$. Assume that $\max_{x\in\com}\|\nabla f(x)\|\le G $. 
	If $\sigma>0$ then the cumulative regret of Algorithm \ref{algo} with \[h_t = \left(\frac{3\kappa\sigma^2}{2(\beta-1)(\kappa_\beta L)^2} \right)^{\frac{1}{2\beta}} t^{-\frac{1}{2\beta}},~~~\eta_t=\frac{2}{\alpha t},~~~t= 1,\dots,T~~~~~~~~\]  satisfies  
	\begin{equation}
	\label{eq:bound0}
	\forall x\in \com: \ \sum_{t=1}^T \mathbb{E} [ f(x_t) - f(x)]\le \frac{1}{\alpha}
	\left(
d^2 \Big(A_1 T^{{1}/{\beta}}+A_2 \Big)
+  A_3  d \log T
\right),
	\end{equation}
	where $A_1={3\beta}(\kappa\sigma^2)^{\frac{\beta-1}{\beta}}
(\kappa_\beta L)^{\frac{2}{\beta}}$, $A_2=\bar c {\bar L^2} (\sigma/L)^{\frac{2}{\beta}}+{9\kappa G^2}/{d}$ with constant $\bar c>0$ depending only on $\beta$, and $A_3 = 9 \kappa G^2$.
	The optimization error of averaged estimator ${\bar x}_T=\frac{1}{T}\sum_{t=1}^T x_t$ satisfies  
	\begin{equation}
	\label{eq:bound1}
	\mathbb{E} [ f({\bar x}_T) - f(x^*)]\leq \frac{1}{\alpha}
\left(  d^2\left(\frac{A_1}{T^{\frac{\beta-1}{\beta}}}+ \frac{A_2}{T} \right)+ A_3 \,\frac{d\log T}{T} \right),
		\end{equation}
		where
	$x^*= \argmin_{x\in\com}f(x)$. 
			If $\sigma=0$, then the cumulative regret and the optimization error of Algorithm \ref{algo} with any $h_t$ chosen small enough and $\eta_t=\frac{2}{\alpha t}$ satisfy the bounds \eqref{eq:bound0} and \eqref{eq:bound1}, respectively, with $A_1=0, A_2={9\kappa G^2}/{d}$ and $A_3=10\kappa G^2$. 
\end{restatable}
\begin{proof}[Proof sketch]
We use the definition of Algorithm \ref{algo} and strong convexity of $f$ to obtain an upper bound for $\sum_{t=1}^{T}\mathbb{E} [ f(x_t) - f(x)  \hspace{.03truecm}|\hspace{.03truecm} x_t]$, which depends on the bias term $\sum_{t=1}^{T}\|\mathbb{E} [{\hat g_t}  \hspace{.03truecm}|\hspace{.03truecm} x_t]-\nabla f(x_t)\|$ and on the stochastic error term $\sum_{t=1}^{T}\mathbb{E} [\|{\hat g}_t\|^2]$. 
By substituting $h_{t}$ (that is derived from balancing the two terms) and $\eta_{t}$ in Lemmas \ref{lem:1} and \ref{lem:2unco}  we obtain upper bounds for $\sum_{t=1}^{T}\|\mathbb{E} [{\hat g_t}  \hspace{.03truecm}|\hspace{.03truecm} x_t]-\nabla f(x_t)\|$ and $\sum_{t=1}^{T}\mathbb{E} [\|{\hat g}_t\|^2]$ that imply the desired upper bound for $\sum_{t=1}^{T}\mathbb{E} [ f(x_t) - f(x)  \hspace{.03truecm}|\hspace{.03truecm} x_t]$ due to a recursive argument in the spirit of \cite{bartlett-hazan-rakhlin}.
\end{proof}

In the non-noisy case ($\sigma=0$) we get the rate $\frac{d}{\alpha} \log T$
for the cumulative regret,  and  $\frac{d}{\alpha} \frac{\log T}{T}$ for the optimization error. 
In what concerns the optimization error, this rate is not optimal since one can achieve much faster rate under strong convexity \cite{NS17}. However,  for the cumulative regret in our derivative-free setting it remains an open question whether the result of Theorem \ref{thm:1}
can be improved. Previous papers on derivative-free online methods with no noise \citep{agarwal2010,duchi2015,flaxman2004} provide slower rates than $({d}/{\alpha}) \log T$. The best known so far is 
$({d^2}/{\alpha})\log T$, cf. \cite[Corollary 5]{agarwal2010}.
We may also notice that the cumulative regret bounds of Theorem \ref{thm:1} trivially extend to the case when we query functions $f_t$ depending on $t$ rather than a single $f$. Another immediate fact is that on the r.h.s. of inequalities \eqref{eq:bound0} and \eqref{eq:bound1} we can take the minimum with $GBT$ and $GB$, respectively, where $B$ is the Euclidean diameter of $\com$. Finally, the factor $\log T$ in the bounds for the optimization error can be eliminated by considering averaging from $T/2$ to $T$ rather than from 1 to $T$, in the spirit of \cite{rakhlin2011making}. We refer to Appendix \ref{app:C} for the details and proofs of these facts.

We now study the performance of Algorithm \ref{algo} when $\com=\mathbb{R}^d$. In this case we make the following choice for the parameters $h_t$ and $\eta_t$ in Algorithm \ref{algo}:
\begin{equation}
\begin{aligned}
h_t &=T^{-\frac{1}{2\beta}},~~~\eta_t = \frac{1}{\alpha T},~~~t= 1,\dots,T_0, \\
h_t &= t^{-\frac{1}{2\beta}},~~~~~\eta_t = \frac{2}{\alpha t},~~~~t= T_0+ 1,\dots,T,
\end{aligned}
\label{eq:hhh}
\end{equation}
where $T_0 = \max \left\{k \geq 0 : C_1 {\bar L}^2 d > {\alpha^2 k}/{2} \right\}$ and $C_1$ is a positive constant\footnote{If $T_0=0$ the algorithm does not use \eqref{eq:hhh}. Assumptions of Theorem \ref{thm:2} are such that condition $T>T_0$ holds.} depending only on the kernel $K(\cdot)$ (this is defined in the proof of Theorem \ref{thm:2} in Appendix \ref{app:B}) and recall ${\bar L}$ is the Lipschitz constant on the gradient $\nabla f$. Finally, define the estimator
\begin{equation}
\label{eq:ppmean}
{\bar x}_{T_0,T} = \frac{1}{T-T_0} \sum_{t=T_0+1}^T x_t.
\end{equation}
\begin{restatable}{theorem}{fiftheorem}{\em (Upper Bounds, Unconstrained Case.)}
\label{thm:2}
	Let $f\in {\cal F} _{\alpha,\beta}(L)$ with $\alpha, L>0$ and $\beta \geq 2$. Let Assumptions \ref{ass1} and \ref{ass:lip} hold.  Assume also that  $\alpha > \sqrt{{C_*d}/{T}}$, where $C_* > 72 \kappa {\bar L}^2$. Let $x_t$'s be the updates of Algorithm \ref{algo} with $\com=\mathbb{R}^d$, $h_t$ and $\eta_t$ as in \eqref{eq:hhh} and a non-random $x_1 \in \mathbb{R}^d$.
	Then the estimator defined by \eqref{eq:ppmean} satisfies
\begin{equation}
\label{eq:pmean}
\mathbb{E}  [f({\bar x}_{T_0,T}) - f(x^*)] \leq 
C \kappa {\bar L}^2 \frac{d}{\alpha T} \|x_1- x^*\|^2
 + C
\frac{d^2}{\alpha} \left(
(\kappa_\beta L)^2 + \kappa \big({\bar L}^2
+ \sigma^2\big)
\right) 
T^{-\frac{\beta-1}{\beta}}
\end{equation}
where $C>0$ is a constant  depending only on $\beta$ 
and $x^*=\argmin_{x\in \mathbb{R}^d}f(x)$.	
\end{restatable}
\begin{proof}[Proof sketch]
As in the proof of Theorem \ref{thm:1}, we apply Lemmas \ref{lem:1} and \ref{lem:2unco}. But we can only use Lemma \ref{lem:2unco}(i) and not Lemma \ref{lem:2unco}(ii) and thus the bound on the stochastic error now involves  $\norm{x_{t}-x^{*}}^{2}$. So, after taking expectations, we need to control an additional term containing  $r_t=\mathbb{E}[\norm{x_{t}-x^{*}}^{2}]$. However, the issue concerns only small $t$ ($t\le T_0\sim d^2/\alpha$) since for bigger $t$ this term is compensated due to the strong convexity with parameter 
$\alpha > \sqrt{{C_*d}/{T}}$. 
This motivates the method where we use the first $T_0$ iterations to get a suitably good (but not rate optimal) bound on $r_{T_0+1}$ and then proceed analogously to Theorem \ref{thm:1} for iterations $t\ge T_0+1$. 
\end{proof}

\section{Estimation of $f(x^*)$}
\label{sec:sec4}
In this section, we apply the above results to estimation of the minimum value  $f(x^*)=\min_{x\in\com}f(x)$ for functions $f$ in the class $\mathcal{F}_{\alpha,\beta}(L)$. The literature related to this problem assumes that $x_t$'s are either i.i.d. with density bounded away from zero on its support \cite{tsy1990} or $x_t$'s are chosen sequentially \cite{mokkadem,belitser}. In the fist case, from the results in \cite{tsy1990} one can deduce that $f(x^*)$ cannot be estimated better than at the slow rate $T^{-\beta/(2\beta+d)}$. 
For the second case, which is our setting, the best result so far is obtained in \cite{belitser}.
The estimator of $f(x
^*)$ in \cite{belitser} is defined via a multi-stage procedure whose complexity increases exponentially with the dimension $d$ and it is shown to achieve  (asymptotically, for $T$ greater than an exponent of $d$) the $c(d, \alpha)/\sqrt{T}$ rate for functions in $\mathcal{F}_{\alpha,\beta}(L)$ with $\beta> 2$. Here, $c(d,\alpha)$ is some constant depending on $d$ and $\alpha$ in an unspecified way.  

Observe that $f(\bar x_T)$ is not an estimator since it depends on the unknown $f$, so Theorem \ref{thm:1} does not provide a result about estimation of $f(x^*)$. 
In this section, we show that using the computationally simple Algorithm \ref{algo} and making one more query per step (that is, having three queries per step in total) allows us to achieve the $1/\sqrt{T}$ rate for all $\beta\ge 2$ with no dependency on the dimension in the main term. Note that the $1/\sqrt{T}$ rate cannot be improved. Indeed, one cannot estimate $f(x^*)$ with  a better rate even using the ideal but non-realizable oracle that makes all queries at point $x^*$. {That is, even if $x^*$ is known and we sample $T$ times $f(x^*) + \xi_t$ with independent centered variables  $\xi_t$, the error is still of the order $1/\sqrt{T}$.}


In order to construct our estimator, at any step $t$ of Algorithm \ref{algo} we make along with $y_t$ and $y_t'$ the third query $y''_t=f(x_t)+\xi''_t$, where $\xi''_t$ is some noise and $x_t$ are the updates of Algorithm \ref{algo}. We estimate $f(x^*)$ by 
$
\hat M = \frac{1}{T}\sum_{t=1}^T y''_t.
$
The properties of estimator $\hat M$ are summarized in the next theorem, which is an immediate corollary of Theorem \ref{thm:1}.
\begin{restatable}{theorem}{fththeorem}
\label{thm:min}
	Let the assumptions of Theorem \ref{thm:1} be satisfied. Let $\sigma>0$ and assume that $(\xi''_t)_{t=1}^T$ are independent random variables with $\mathbb{E}[\xi''_t]=0$ and $\mathbb{E}[(\xi''_t)^2]\le \sigma^2$ for $t=1,\dots, T$.
	If $f$ attains its minimum at point
	$x^*\in \com$, then 
	\begin{equation}
	\label{eq:boundM}
	\mathbb{E} \lvert \hat M - f(x^*) \rvert \leq \frac{\sigma}{T^{\frac{1}{2}}}+ \frac{1}{\alpha}
\left(  d^2\left(\frac{A_1}{T^{\frac{\beta-1}{\beta}}}+ \frac{A_2}{T}\right)+ 
A_3\frac{d\log T}{T} \right).
		\end{equation}
		\end{restatable}
\begin{remark}\label{remark:min}
With three queries per step, the risk (error) of the oracle that makes all queries at point $x^*$ does not exceed  $\sigma/\sqrt{3T}$. Thus, for $\beta>2$ the estimator $\hat M$ achieves asymptotically as $T\to \infty$ the oracle risk up to a numerical constant factor. We do not obtain such a sharp property for $\beta=2$, in which case the remainder term in Theorem \ref{thm:min} accounting for the accuracy of Algorithm \ref{algo} is of the same order as the main term $\sigma/\sqrt{T}$. 
\end{remark}
Note that in Theorem \ref{thm:min} the noises $(\xi''_t)_{t=1}^T$ are assumed to be independent and zero mean random variables, which is essential to obtain the $1/\sqrt{T}$ rate. Nevertheless, we do not require independence between the noises $(\xi''_t)_{t=1}^T$ and the noises in the other two queries $(\xi_t)_{t=1}^T$ and $(\xi'_t)_{t=1}^T$. 
Another interesting point is that for $\beta=2$ the third query is not needed and $f(x^*)$ is estimated with the $1/\sqrt{T}$ rate either by 
$
\hat M = \frac{1}{T}\sum_{t=1}^T y_t
$
or by 
$
\hat M = \frac{1}{T}\sum_{t=1}^T y'_t.
$
This is an easy consequence of the above argument, the property \eqref{lem6.4_2} -- see Lemma \ref{lem6.4} in the appendix -- which is specific for the case $\beta=2$, and the fact that the optimal choice of $h_t$ is of order $t^{-1/4}$ for $\beta=2$.

\section{Improved bounds for $\beta=2$}
\label{sec:beta2} 

In this section,  we consider the case $\beta=2$ and obtain improved bounds that scale as $d$ rather than $d^2$ with the dimension in the constrained optimization setting analogous to Theorem \ref{thm:1}. First note that for $\beta=2$ we can simplify the algorithm.  The use of kernel $K$ is redundant when $\beta=2$, and therefore in this section we define the approximate gradient as 
\begin{equation}\label{gt}
   \hat g_t = \frac{d}{2h_t} (y_t- y'_t) \zeta_t, 
\end{equation}
where $y_t=f(x_t+h_t \zeta_t)+\xi_t$ and $y'_t=f(x_t-h_t  \zeta_t) + \xi_t'$. A well-known observation that goes back to \cite{NY1983} consists in the fact that $\hat g_t$ defined in \eqref{gt} is an unbiased estimator of the gradient at point $x_t$ of the surrogate function $\hat f_t$ defined by 
\[
\hat f_t(x)= \mathbb{E} f(x+h_t \tilde \zeta),\quad \forall x\in\mathbb{R}^d,
\]
where the expectation  $\mathbb{E}$ is taken   with respect to the random vector $\tilde \zeta$ uniformly distributed on the unit ball $B_d=\{u\in \mathbb{R}^d:\|u\|\le 1 \}$.  The properties of the surrogate $\hat f_t$  are described in Lemmas \ref{lem6.2} and \ref{lem6.4} presented in the appendix.

The improvement in the rate that we get for $\beta=2$ is due to the fact that we can consider Algorithm \ref{algo} with  $\hat g_t$ defined in \eqref{gt} as the SGD for the surrogate function. Then the bias of approximating $f$ by $\hat f_t$ scales as $h_t^2$, which is smaller than the squared bias of 
approximating the gradient arising in the proof of Theorem \ref{thm:1} that scales as $d^2h_t^{2(\beta-1)}=d^2h_t^2$ when $\beta=2$. On the other hand, the stochastic variability terms are the same for both methods of proof. This explains the gain in dependency on $d$. However, this technique does not work for $\beta>2$ since then the error of approximating $f$ by $\hat f_t$, which is of the order $h_t
^\beta$ (with $h_t$ small), becomes too large compared to the bias $d^2h_t^{2(\beta-1)}$ of Theorem \ref{thm:1}. 
\begin{restatable}{theorem}{eiththeorem}
\label{th6.3}
Let $f\in \mathcal{F}_{\alpha, 2}(L)$ with $\alpha,L>0$.
	Let Assumption \ref{ass1}  hold and let $\com$ be a convex compact subset  of $\mathbb{R}^d$. Assume that $\max_{x\in\com}\|\nabla f(x)\|\le G $. 
If $\sigma>0$ then for the updates $x_t$ as in item 3 of Algorithm \ref{algo} with $\hat g_t$ defined in \eqref{gt} and parameters $h_t= \left(\frac{3d^2\sigma^2}{4L\alpha t+9L^2d^2}\right)^{1/4}$ and $\eta_t = \frac{1}{\alpha t}$ we have 
\begin{equation}\label{eq1:th6.3}
\forall x\in \Theta: \ \ \mathbb{E} \sum_{t=1}^{T} \big(f(x_t) - f(x)\big)  
\le  \min\left(GBT, 2\sqrt{3L}\sigma
\frac{d}{\sqrt{\alpha}} \sqrt{T}+ A_4\frac{d^2}{\alpha} \log T\right), 
\end{equation}
where $B$ is the Euclidean diameter of $\com$ and 
$A_4=
6.5 L\sigma+ 22 G^2/d$. 
Moreover, if $x^*= \argmin_{x\in\com}f(x)$ the optimization error of averaged estimator ${\bar x}_T=\frac{1}{T}\sum_{t=1}^T x_t$ is bounded as  
	\begin{equation}
	\label{eq2:th6.3}
	\mathbb{E} [ f({\bar x}_T) - f(x^*)]\leq 
	\min\left(GB, 2\sqrt{3L}\sigma
 \frac{d}{\sqrt{\alpha T}}+ A_4 \frac{d^2}{\alpha} \frac{\log T}{T}\right).
		\end{equation}
Finally, if $\sigma=0$, then the cumulative regret of the same procedure $x_t$ with any $h_t>0$ chosen small enough and $\eta_t=\frac{1}{\alpha t}$ and the optimization error of its averaged version are of the order $\frac{d^2}{\alpha} \log T$ and $\frac{d^2}{\alpha} \frac{\log T}{T}$, respectively.                         \end{restatable}
Note that the terms $\frac{d^2}{\alpha} \log T$ and $\frac{d^2}{\alpha} \frac{\log T}{T}$ appearing in these bounds can be improved to $\frac{d}{\alpha} \log T$ and $\frac{d}{\alpha} \frac{\log T}{T}$ at the expense of assuming that the norm $\|\nabla f\|$ is uniformly bounded by $G$ not only on $\com$ but also on a large enough Euclidean neighborhood of $\com$. Moreover, the $\log T$ factor in the bounds for the optimization error can be eliminated by considering averaging from $T/2$ to $T$ rather than from 1 to $T$ in the spirit of  \cite{rakhlin2011making}. We refer to Appendix \ref{app:C} 
for the details and proofs of these facts. A major conclusion is that, when $\sigma>0$ and we consider the optimization error,  those terms are negligible with respect to $d/\sqrt{\alpha T}$ and thus an attainable rate is $\min(1,d/\sqrt{\alpha T})$. 

We close this section by noting, in connection with the bandit setting, that the bound \eqref{eq1:th6.3} extends straightforwardly (up to a change in numerical constants) to the cumulative regret of the form $\mathbb{E}\sum_{t=1}^{T} \big(f_t(x_t\pm h_t\zeta_t) - f_t(x)\big)$, where the losses are measured at the query points and $f$ depends on $t$. This fact follows immediately from the proof of Theorem \ref{th6.3} presented in the appendix and the property \eqref{lem6.4_2}, see Lemma \ref{lem6.4} in the appendix.

\section{Lower bound}
\label{sec:LB}
In this section we prove a minimax lower bound on the optimization error over all sequential strategies that allow the query points depend on the past. For $t=1,\dots, T$, we assume that $y_t=f(z_t)+\xi_t$ and we consider strategies of choosing the query points such that $z_1\in \mathbb{R}^d$ is a random variable and $z_t= \Phi_t(z_1,y_1,\dots, z_{t-1},y_{t-1},\boldsymbol{\zeta}_t)$ for $t\ge2$, where $\Phi_t$'s are measurable functions with values in $\mathbb{R}^d$, and $\{\boldsymbol{\zeta}_t\}$
is a sequence of random variables with values in some measurable space $(\mathcal Z, \mathcal U)$ (a randomizing sequence) satisfying the condition that $\boldsymbol{\zeta}_t$ is independent of $(z_1,y_1,\dots, z_{t-1},y_{t-1})$. We denote by $\Pi_T$ the set of all such strategies. The noises $\xi_1,\dots,\xi_T$ are assumed in this section to be independent with cumulative distribution function $F$ satisfying the condition
\begin{equation}\label{distribution}
\int \log\big(dF(u)/dF(u+v)\big)dF(u)\leq I_{0}v^{2}, \quad |v|< v_{0},    
\end{equation}
for some $0<I_{0}<\infty$, $0<v_{0}\leq \infty$, and such that $\xi_t$ is independent of $(z_1,y_1,\dots, z_{t-1},y_{t-1},\boldsymbol{\zeta}_t)$. {Using the second order expansion of the logarithm w.r.t. $v$, one can verify that this assumption is satisfied when $F$ has a smooth enough density with finite Fisher information.} For example, for Gaussian distribution $F$ 
this condition holds with $v_0=\infty$. Note that the class $\Pi_T$ includes the sequential strategy of Algorithm \ref{algo} that corresponds to taking $T$ as an even number, and choosing $z_t=x_t+\zeta_tr_t$ and $z_t=x_t-\zeta_tr_t$  for even $t$ and odd $t$, respectively. 

\begin{restatable}{theorem}{lowerB}
\label{lb}
Let $\com=\{x\in\mathbb{R}^d: \norm{x}\le 1\}$. For $\alpha, L>0$,$ \beta\ge 2$, let  $\mathcal{F}'_{\alpha,\beta}$ denote the set of functions $f$ that attain their minimum over $\mathbb{R}^d$ in $\com$ and belong to $\mathcal{F}_{\alpha,\beta}(L)\cap \{f: \max_{x\in\com}\|\nabla f(x)\|\le G\}$,  where $G> 2\alpha$. Then for any strategy in the class $\Pi_T$ we have 
\begin{equation}\label{eq1:lb}
\sup_{f \in \mathcal{F}'_{\alpha,\beta}}\mathbb{E}\big[f(z_T)-\min_{x}f(x)\big]\geq C\min\Big(\max(\alpha, T^{-1/2+1/\beta}), \frac{d}{\sqrt{T}}, \,\frac{d}{\alpha}T^{-\frac{\beta-1}{\beta}}\Big),
\end{equation}
and
\begin{equation}\label{eq2:lb}
\sup_{f \in \mathcal{F}'_{\alpha,\beta}}\mathbb{E}\big[\norm{z_{T}-x^{*}(f)}^{2}\big]\geq C\min\Big(1, \frac{d}{T^{\frac{1}{\beta}}}, \,\frac{d}{\alpha^{2}}T^{-\frac{\beta-1}{\beta}}\Big),
\end{equation}
 where  $C>0$ is a constant that does not depend  of $T,d$, and $\alpha$, and $x^*(f)$ is the minimizer of $f$ on $\com$. 
\end{restatable}
The proof is given in  Appendix \ref{app:B}. {It extends the proof technique of Polyak and Tsybakov [28], by applying it to more than two probe functions. The proof takes into account dependency on the dimension $d$, and on $\alpha$. The final result is obtained by applying Assouad’s Lemma, see e.g. \cite{Tsybakov09}.}


%
%
We stress that the condition $G> 2\alpha$ in this theorem is necessary. It should always hold if the intersection $\mathcal{F}_{\alpha,\beta}(L)\cap \{f: \max_{x\in\com}\|\nabla f(x)\|\le G\}$ is not empty. Notice also that the threshold $T^{-1/2+1/\beta}$ on the strong convexity parameter $\alpha$ plays an important role in bounds \eqref{eq1:lb} and \eqref{eq2:lb}. Indeed, for $\alpha$ below this threshold, the bounds start to be independent of $\alpha$.  Moreover, in this regime, the rate of \eqref{eq1:lb} becomes $\min(T^{1/\beta},d)/\sqrt{T}$, which is asymptotically  $d/\sqrt{T}$ and thus not better as function of $T$ than the rate attained for 
zero-order minimization of simply convex functions \cite{agarwal2011,belloni2015escaping}. Intuitively, it seems reasonable that $\alpha$-strong convexity should be of no added value for very small $\alpha$. Theorem \ref{lb} allows us to quantify exactly how small such $\alpha$ should be. Also, quite naturally, the threshold becomes smaller when the smoothness $\beta$ increases.\\
Finally note that for $\beta=2$ the lower bounds \eqref{eq1:lb} and \eqref{eq2:lb} are, in the interesting regime of large enough $T$, of order $d/(\max(\alpha, 1)\sqrt{T})$ and $d/(\max(\alpha^2, 1)\sqrt{T})$, respectively. This highlights the near minimax optimal properties of Algorithm \ref{algo} in the setting of Theorem \ref{th6.3}.


\section{Discussion and related work}
\label{sec:PW}
There is a great deal of attention to zero-order feedback stochastic optimization and convex bandits problems in the recent literature. Several settings are studied: (i) deterministic in the sense that the queries contain no random noise and we query functions $f_t$ depending on $t$ rather than $f$ where $f_t$ are Lipschitz or 2-smooth \cite{flaxman2004,agarwal2010,Nesterov2011,NS17,saha2011,Shamir17}; (ii) stochastic with two-point feedback where the two noisy evaluations are obtained with the same noise and the noisy functions are Lipschitz or 2-smooth \cite{Nesterov2011,NS17,duchi2015} (this setting does not differ much from (i) in terms of the analysis and the results); (iii) stochastic, where the noises $\xi_i$ are independent zero-mean random variables  \cite{Fabian,PT90,dippon,agarwal2011,Shamir13,BP2016, jamieson2012,bartlett2018simple,Locatelli}. In this paper, we considered a setting, which is more general than (iii) by allowing for adversarial noise (no independence or zero-mean assumption in contrast to (iii), no Lipschitz assumption  in contrast to settings (i) and (ii)), which are both covered by our results when the noise is set to zero. 

One part of our results are bounds on the cumulative regret, cf. \eqref{eq:bound0} and \eqref{eq1:th6.3}. We emphasize that they remain trivially valid if the queries are from $f_t$ depending on $t$ instead of $f$, and thus cover the setting (i). To the best of our knowledge, there were no such results in this setting previously, except for \cite{BP2016} that gives  bounds with suboptimal dependency on $T$ in the case of classical (non-adversarial) noise. In the non-noisy case, we get bounds on the cumulative regret with faster rates than previously known for the setting (i). It remains an open question whether these bounds can be improved. 

The second part of our results dealing with the optimization error $\mathbb{E}[f(\bar x_T)-f(x^*)]$ is closely related to the work on derivative-free stochastic optimization under strong convexity and smoothness assumptions initiated in \cite{Fabian,PT90} and more recently developed in \cite{dippon,jamieson2012,Shamir13,BP2016}. 
It was shown in \cite{PT90} that the minimax optimal rate for $f\in\mathcal{F}_{\alpha,\beta}(L)$ scales as $c(\alpha,d)T^{-{(\beta-1)}/{\beta}}$, where $c(\alpha,d)$ is an unspecified function of $\alpha$ and $d$ (for $d=1$ an upper bound of the same order was earlier established in 
\cite{Fabian}). The issue of establishing non-asymptotic fundamental limits as function of the main parameters of the problem ($\alpha$, $d$ and $T$) was first addressed in \cite{jamieson2012} giving a lower bound $\Omega(\sqrt{d/T})$ for $\beta=2$. This was improved to $\Omega(d/\sqrt{T})$ when $\alpha\asymp 1$ by Shamir \cite{Shamir13} who conjectured that the rate $d/\sqrt{T}$ is optimal for $\beta=2$, which indeed follows from our Theorem \ref{th6.3} (although \cite{Shamir13} claims the optimality as proved fact by referring to results in \cite{agarwal2010}, such results cannot be applied in setting (iii) because the noise cannot be considered as Lipschitz).  
A result similar to Theorem \ref{th6.3} is stated without proof in Bach and Perchet \cite[Proposition 7]{BP2016} but not for the cumilative regret and with  a suboptimal rate in the non-noisy case. For integer $\beta\ge 3$, Bach and Perchet \cite{BP2016} present explicit upper bounds as functions of $\alpha$, $d$ and $T$ with, however, suboptimal dependency on $T$ except for their Proposition 8 that is problematic (see Appendix \ref{app:D} for the details).  Finally, by slightly modifying the proof of Theorem \ref{thm:1} we get that the estimation risk $\mathbb{E}\big[\norm{\bar x_{T}-x^{*}}^{2}\big]$ is $O((d^2/\alpha
^2)T^{-{(\beta-1)}/{\beta}})$, which is to within factor $d$ of the main term in the lower bound \eqref{eq2:lb} (see Appendix \ref{app:C} for details).

The lower bound in Theorem  \ref{lb} is, to the best of our knowledge, the first result providing non-asymptotic fundamental limits 
under general configuration of $\alpha$, $d$ and $T$. The known lower bounds \cite{PT90,jamieson2012,Shamir13} either give no explicit dependency on  $\alpha$ and $d$,
or treat the special case $\beta=2$ and $\alpha\asymp 1$. Moreover, as an interesting consequence of our lower bound we find that, for small strong convexity parameter $\alpha$ (namely, below the $T^{-1/2+1/\beta}$ threshold), the best achievable rate cannot be substantially faster 
than for simply convex functions, at least for moderate dimensions. Indeed, for such small $\alpha$, our lower bound is asymptotically $\Omega(d/\sqrt{T})$ independently of the smoothness index $\beta$ and on $\alpha$, while the achievable rate for convex functions is shown to be $d^{16}/\sqrt{T}$ in \cite{agarwal2011} and improved to $d^{3.75}/\sqrt{T}$ in  \cite{belloni2015escaping} (both up to log-factors). The gap here is only in the dependency on the dimension. Our results imply that for $\alpha$ above the $T^{-1/2+1/\beta}$ threshold, the gap between upper and lower bounds is much smaller. Thus, our upper bounds in this regime scale as $(d^2/\alpha)T^{-{(\beta-1)}/{\beta}}$ while the lower bound of 
Theorem  \ref{lb} is of the order $\Omega\big((d/\alpha)T^{-{(\beta-1)}/{\beta}}\big)$; moreover for $\beta=2$, upper and lower bounds match in the dependency on $d$.

We hope that our work will stimulate further study at the intersection of zero-order optimization and convex bandits in machine learning. An important open problem is to study novel algorithms which match our lower bound simultaneously in all main parameters. For example a class of algorithms worth exploring are those using memory of the gradient in the spirit of Nesterov accelerated method. Yet another important open problem is to study lower bounds for the regret in our setting. Finally, it would be valuable to study extensions of our work to locally strongly convex functions.



\section*{Broader impact}
The present work improves our understanding of zero-order optimization methods in specific scenarios in which the underlying function we wish to optimize has certain regularity properties. We believe that a solid theoretical foundation is beneficial to the development of practical machine learning and statistical methods. We expect no direct or indirect ethical risks from our research.

\begin{ack}
We would like to thank Francis Bach, Vianney Perchet, Saverio Salzo, and Ohad Shamir for helpful discussions. The first and second authors were  partially supported by SAP SE. The research of A.B. Tsybakov is supported by a grant of the French National Research Agency (ANR), “Investissements d'Avenir” (LabEx Ecodec/ANR-11-LABX-0047).
\end{ack}

\bibliography{biblio}
\bibliographystyle{plainnat}

\newpage

\begin{appendices}

\section*{Supplementary material}

The supplementary material is organized as follows. In Appendix~\ref{app:A} we provide some auxiliary results, including those stated in Section \ref{sec:BM} above. In Appendix~\ref{app:B} we give proofs of the results which were only stated or whose proof was only sketched in the paper. For reader's convenience all such results are restated below. Appendix \ref{app:D} contains some comments on previous results in \cite{BP2016}. 
Finally, in Appendix \ref{app:C} we present refined versions of Theorems \ref{thm:1} and \ref{th6.3}. 

\section{Auxiliary results}
\label{app:A}

\firstlemma*
\begin{proof}
To lighten the presentation and without loss of generality we drop the lower script ``$t$'' in all quantities. Using the Taylor expansion we have
$$
f(x+ h r\zeta) = f(x)+ \langle \nabla f(x), h r\zeta \rangle + \sum_{2\leq |m|\leq \ell} \frac{(rh)^{|m|}}{m!} D^{(m)}f(x) \zeta^m + R(hr\zeta),
$$
where  by assumption $|R(hr \zeta)| \leq L\|hr\zeta\|^\beta = L |r|^{\beta}h^{\beta}$. Thus,
$$
\mathbb{E} [\hat g|x] = \frac{d}{h} \mathbb{E} \Big[\Big( \langle \nabla f(x), hr \zeta \rangle +  \hspace{-.3truecm}\sum_{2\leq |m|\leq \ell, |m| \,\text{odd}} \frac{(rh)^{|m|}}{m!} D^{(m)}f(x) \zeta^m + \frac{R(hr\zeta) - R(-hr\zeta)}{2} \Big) \zeta K(r)\Big].
$$
Since $\zeta$ is uniformly distributed on the unit sphere we have $\mathbb{E} [\zeta \zeta^\top] = (1/d) I_{d \times d}$, where $I_{d \times d}$ is the identity matrix. Therefore,  
$$
\mathbb{E} \Big[\frac{d}{h} \langle \nabla f(x), h \zeta \rangle \zeta \Big] = \nabla f(x).
$$
As $\int r^{|m|} K(r) dr = 0$ for $2\le |m| \leq 
\ell$ and $\int r K(r) dr = 1$ we conclude that 
\begin{align*}
  \|\mathbb{E} [{\hat g} \hspace{.03truecm}|\hspace{.03truecm} x] -  \nabla f(x)\| & = \frac{d}{2h}\| \mathbb{E} \big[ \big(R(hr\zeta) - R(-hr\zeta)\big) \zeta K(r)\big]\| \\
  &\le  \frac{d}{2h} \mathbb{E} \big[|R(hr\zeta) - R(-hr\zeta)|\, |K(r)|\big] \leq  \kappa_\beta L  d h^{\beta-1}.  
\end{align*}
\end{proof}
\sndlemma* 
\begin{proof}
We have 
\begin{align*}
 \|{\hat g}\|^2 &=  \frac{d^2}{4h^2} \big\|\big(f(x+h r\zeta) - f(x-hr\zeta) + \xi - \xi'\big)\zeta K(r) \big\|^2 \\
 &= \frac{d^2}{4h^2} \big(f(x+h r\zeta) - f(x-hr\zeta) + \xi - \xi'\big)^2 K^2(r).   
\end{align*}
Using the inequality $(a+b+c)^2 \leq 3 (a^2+b^2+c^2)$ we get
\begin{equation}\label{var1}
   \mathbb{E} [\|\hat g\|^2 \hspace{.03truecm}|\hspace{.03truecm} x] \leq \frac{3d^2}{4h^2} \left( \mathbb{E}\big[\big(f(x+hr \zeta) - f(x-hr\zeta) \big)^2K^2(r)\big] + 2 \kappa\sigma^2 \right). 
\end{equation}
Here, 
\begin{eqnarray}
\nonumber
\big(f(x+hr \zeta) - f(x-hr\zeta)\big)^2  &=&  \big(f(x+hr \zeta) - f(x-hr\zeta)  \pm f(x) \pm 2 \langle\nabla f(x),h r\zeta\rangle\big)^2~~~ \\
\nonumber
& \leq & 3 \bigg\{  \Big(f(x+hr \zeta) {-}  f(x) -  \langle \nabla f(x),h r\zeta\rangle \Big)^2 
\\ 
\nonumber 
&~&+  \Big(f(x-hr \zeta) - f(x) - \langle \nabla f(x),-h r\zeta\rangle\Big)^2 + 4\langle\nabla f(x),h r\zeta\rangle^2 \bigg\}~~~~~~\\
\nonumber
&\leq & 3 \left( \frac{{\bar L}^2}{2} \|h r\zeta\|^4 + 
4\langle\nabla f(x),h r\zeta\rangle^2 \right), \label{eq:bbb}
\end{eqnarray}
where the last inequality follows from standard properties of convex functions with Lipschitz continuous gradient, see e.g.,  \cite[Lemma~3.4]{bubeck2014}. Taking the expectation and using the fact that $\mathbb{E} [\zeta \zeta^\top] = (1/d) I_{d \times d}$ we obtain
\begin{equation}\label{razn}
  \mathbb{E}[(f(x+hr \zeta) - f(x-hr\zeta) )^2K^2(r)]
 \leq 3 \kappa \left( \frac{{\bar L}^2 h^4}{2}  + 
\frac{4  h ^2}{d} \|\nabla f(x) \|^2  \right).  
\end{equation}
To prove part (i) of the lemma, it is enough to combine \eqref{var1}, \eqref{razn} and the inequality $\|\nabla f(x) \|\le {\bar L}\|x-x^*\|$ that follows from the Lipschitz gradient assumption and the fact that $\nabla f(x^*) = 0$. Next, under the assumptions of part (ii) of the lemma we get analogously to \eqref{eq:bbb} that 
$$
\big(f(x+hr \zeta) - f(x-hr\zeta)\big)^2\leq  3 \left(2L^2 \|h r\zeta\|^4 + 
4\langle\nabla f(x),h r\zeta\rangle^2 \right).
$$
This yields inequality \eqref{razn} with the only difference that $\bar L^2/2$ is replaced by $2L^2$. Together with \eqref{var1}, it implies the result.
\end{proof}
\begin{lemma}
Let $f$ be Lipschitz continuous with constant $G>0$ in a Euclidean $h_t$-neighborhood of the set $\com$, and let Assumption \ref{ass1} (i) hold. Let ${\hat g}_t$ and $x_t$ be defined by Algorithm \ref{algo}.
Then
$$
\mathbb{E} [\|{\hat g}_t \|^2 \hspace{.04truecm}|\hspace{.04truecm} x_t] \leq \kappa\Big(C^* G^2 d +     \frac{3d^2}{2h_t^2}\sigma^2\Big),
$$
where $C^*>0$ is a numerical constant and $\kappa=\int K^2(u)du$. 
\label{lem:2}
\end{lemma}
\begin{proof}
We have 
\begin{align*}
 \|{\hat g}\|^2 &=  \frac{d^2}{4h^2} \|(f(x+h r\zeta) - f(x-hr\zeta) + \xi - \xi')\zeta K(r) \|^2 \\
 &= \frac{d^2}{4h^2} (f(x+h r\zeta) - f(x-hr\zeta) + \xi - \xi')^2 K^2(r).   
\end{align*}
Using the inequality $(a+b+c)^2 \leq 3 (a^2+b^2+c^2)$ we get
\begin{equation*}\label{var}
   \mathbb{E} [\|\hat g\|^2 \hspace{.03truecm}|\hspace{.03truecm} x] \leq \frac{3d^2}{4h^2} \left( \mathbb{E}[(f(x+hr \zeta) - f(x-hr\zeta) )^2K^2(r)] + 2 \kappa\sigma^2 \right). 
\end{equation*}
The lemma now follows by using  \cite[Lemma 10]{Shamir17}, which shows by a concentration argument that if $x\in\com$, $r\in [-1,1]$ are fixed, $\zeta$ is uniformly distributed on the unit sphere and $f$ is Lipschitz continuous with constant $G>0$ in a Euclidean $h$-neighborhood of the set $\com$, then
 $$
 \mathbb{E}[(f(x+hr \zeta) - f(x-hr\zeta) )^2] \leq c \frac{(hr)^2 G^2}{d}, 
$$
where $c>0$ is a numerical constant.
\end{proof}

%

\begin{lemma}
\label{lem6.2}
	Let $f(\cdot)$ be a convex function on $\mathbb{R}^d$ and $h_t>0$. Then the following holds.
	\begin{itemize}
		\item[(i)] Function $\hat f_t(\cdot)$ is convex on $\mathbb{R}^d$.
		\item[(ii)] $\hat f_t(x)\ge f(x)$ for all $x\in \mathbb{R}^d$.
		\item[(iii)] Function $\hat f_t(\cdot)$ is differentiable on $\mathbb{R}^d$ and for the conditional expectation given $x_t$ we have
		$$
		\mathbb{E}[\hat g_t| x_t] = \nabla \hat f_t(x_t).
		$$
	\end{itemize}
\end{lemma}
\begin{proof}
	Item (i) is straightforward.  To prove item (ii), consider $g_t \in\partial f(x)$. Then,
	$$\hat f_t(x)\ge	\mathbb{E} \big[f(x)+ h_t\langle g_t, \tilde \zeta \rangle\big]
	= f(x)+ h_t\langle g_t, \mathbb{E} [\tilde \zeta] \rangle = f(x).
	$$
	For item (iii) we refer to \cite[][pg. 350]{NY1983}, or \cite{flaxman2004}. It is based on the fact that for any $x\in \mathbb{R}^d$ using Stokes formula we have
	\begin{align*}
	\nabla \hat f_t(x) &= \frac{1}{V(B_d)h_t^d} \int_{\norm{v}=h_t}f(x+v)\frac{v}{\norm{v}} {\rm d}s_{h_t}(v)= \frac{d}{V(S_d)h_t}\int_{\norm{u}=1}f(x+h_tu)u\, {\rm d}s_1(u) \\
	& = \frac{d}{V(S_d)h_t} \int_{\norm{u}=1}f(x+h_tu)u \,{\rm d}s_1(u) = \mathbb{E} \Big[  \frac{d}{h_t} f(x+h_t \zeta_t)\zeta_t \Big]
	\end{align*}	
	where $V(B_d)$ is the volume of the unit ball $B_d$,  ${\rm d}s_r(\cdot)$ is the element of spherical surface of raduis $r$ in $\mathbb{R}^d$, and $V(S_d)=d V(B_d)$ is the surface area of the unit sphere in $\mathbb{R}^d$. Since $f(x+h_t \zeta_t)\zeta_t $ has the same distribution as $f(x-h_t \zeta_t)(-\zeta_t)$ we also get  
	$$
	\mathbb{E} \Big[  \frac{d\big(f(x+h_t \zeta_t)-f(x-h_t \zeta_t)\big)\zeta_t }{2h_t}\Big] = \nabla \hat f_t(x).
	$$
\end{proof}

\begin{lemma}
\label{lem6.4}
If $f$ is $\alpha$-strongly convex then $ \hat  f_t$ is $\alpha$-strongly convex. 
	If $f\in \mathcal{F}_2(L)$ then for any $x\in\mathbb{R}^d$ and $h_t>0$ we have
	\begin{align}\label{lem6.4_1}
	| \hat f_t(x)- f(x)|  
	\le  Lh_t^{2}. 
	\end{align}
	and
	\begin{align}\label{lem6.4_2}
	| \mathbb{E}  f(x\pm h_t\zeta_t) - f(x) | \le  Lh_t^{2}.
	\end{align}
\end{lemma}
\begin{proof}
Using the fact that $\mathbb{E} [\tilde \zeta] =0$ we have
$$
 |\mathbb{E} \big[f(x+h_t \tilde \zeta)- f(x)\big] | =  
 |\mathbb{E} \big[f(x+h_t \tilde \zeta)- f(x) - \langle\nabla f(x), h_t \tilde \zeta\rangle\big] |  \le  Lh_t^{2} \mathbb{E} [\|\tilde \zeta\|^{2}] \le Lh_t^{2}.
$$
Thus, \eqref{lem6.4_1} follows. The proof of \eqref{lem6.4_2} is analogous.
The $\alpha$-strong convexity of $\hat  f_t$ is equivalent to the relation 
\begin{align}\label{lem6.4_4}
\langle\nabla \hat f_t(x) - \nabla \hat f_t(x'), x-x'\rangle  \ge \alpha \norm{x- x'}^{2}, \quad \forall x,x'\in \mathbb{R}^d,  
\end{align}
which is proved as follows:
\begin{align}\label{lem6.4_5}
\langle\nabla \hat f_t(x) - \nabla \hat f_t(x'), x-x'\rangle & = \langle\mathbb{E} \big[\nabla f(x+h_t\tilde\zeta) - \nabla f(x'+h_t\tilde\zeta)\big], x-x'\rangle  \\ \nonumber
& =  \mathbb{E} \big[\langle \nabla f(x+h_t\tilde\zeta) - \nabla f(x'+h_t\tilde\zeta), (x+h_t\tilde\zeta)-(x'+h_t\tilde\zeta)\rangle \big]  \\
& \ge     \alpha \norm{x- x'}^{2}, \quad \forall x,x'\in \mathbb{R}^d, \nonumber
\end{align}
due to the $\alpha$-strong convexity of $f$.
\end{proof}

\section{Proofs}
\label{app:B}
\thdtheorem*
\begin{proof}
Fix an arbitrary $x \in \Theta$. By the definition of the algorithm, we have $\|x_{t+1} - x\|^2 \leq \|x_t-\eta_t \hat g_t - x\|^2$, which is equivalent to
\begin{equation}
\label{eq:GD} \langle {\hat g}_t , x_t - x\rangle \leq \frac{\|x_t-x\|^2- \|x_{t+1}- x]\|^2}{2 \eta_t} + \frac{\eta_t}{2} \|{\hat g}_t\|^2.
\end{equation}

By the strong convexity assumption we have
\begin{equation}
\label{eq:sconv}
f(x_t) - f(x) \leq \langle \nabla f(x_t),x_t - x\rangle -\frac{\alpha}{2} \|x_t - x\|^2.
\end{equation}
Combining the last two displays and setting $a_t = \|x_t-x\|^2$ we obtain
\begin{eqnarray}
\nonumber
\mathbb{E} [ f(x_t) - f(x)  \hspace{.03truecm}|\hspace{.03truecm} x_t] & \leq & \|\mathbb{E} [{\hat g_t}  \hspace{.03truecm}|\hspace{.03truecm} x_t]-\nabla f(x_t)\|\|x_t - x\| + \frac{1}{2\eta_t}\mathbb{E} [a_t - a_{t+1}  \hspace{.03truecm}|\hspace{.03truecm}x_t] \\
\nonumber
&~& + \frac{\eta_t}{2} \mathbb{E} [\|{\hat g}_t\|^2  \hspace{.03truecm}|\hspace{.03truecm}x_t] -  \frac{\alpha}{2} \mathbb{E} [a_t  \hspace{.03truecm}|\hspace{.03truecm} x_t] \\ \nonumber
& \leq & \kappa_\beta L d h_t^{\beta-1} \|x_t - x\| + \frac{1}{2\eta_t}\mathbb{E} [a_t - a_{t+1}  \hspace{.03truecm}|\hspace{.03truecm}x_t]\\
&~& + 
\frac{\eta_t}{2} \mathbb{E} [\|{\hat g}_t\|^2  \hspace{.03truecm}|\hspace{.03truecm}x_t]\label{eq:main}
-  \frac{\alpha}{2} \mathbb{E} [a_t  \hspace{.03truecm}|\hspace{.03truecm} x_t],
\end{eqnarray}
where the second inequality follows from Lemma \ref{lem:1}.  
As $2ab \leq a^2+b^2$ we have
\begin{equation}\label{eq:decouple}
  d h_t^{\beta-1} \|x_t - x\| \leq \frac{1}{2} \Big( \frac{2\kappa_\beta L}{\alpha} d^2 h_t^{2(\beta-1)}   + \frac{\alpha}{2\kappa_\beta L} \|x_t-x\|^2 \Big).
\end{equation}
We conclude, taking the expectations and
letting 
$r_t = \mathbb{E} [a_t]$, that 
\begin{equation}
\label{eq:A}
\mathbb{E} [ f(x_t) - f(x)]  \leq \frac{r_t - r_{t+1}}{2\eta_t} - \frac{\alpha}{4} r_t + (\kappa_\beta L)^2 \frac{d^2}{\alpha} h_t^{2(\beta-1)} +
\frac{\eta_t}{2} \mathbb{E} [\|{\hat g}_t\|^2]
\end{equation}
Summing both sides over $t$ gives
$$
\sum_{t=1}^T \mathbb{E} [ f(x_t) - f(x)] \leq \frac{1}{2} \sum_{t=1}^T \left(\frac{r_t - r_{t+1}}{\eta_t} - \frac{\alpha}{2} r_t \right) + \sum_{t=1}^T \Big( (\kappa_\beta L)^2\frac{d^2}{\alpha} h_t^{2(\beta-1) }  +
\frac{\eta_t}{2} \mathbb{E} [\|{\hat g}_t\|^2]
\Big).
$$
The first sum on the r.h.s. is smaller than 0 for our choice of $\eta_t = \frac{2}{\alpha t}$. Indeed,
$$
\sum_{t=1}^T \left(\frac{r_t - r_{t+1}}{\eta_t} - \frac{\alpha}{2} r_t \right) \leq r_1\Big(\frac{1}{\eta_1}- \frac{\alpha}{2}\Big) + \sum_{t=2}^T r_t\left(\frac{1}{\eta_t} - \frac{1}{\eta_{t-1}} - \frac{\alpha}{2}\right)  = 0.
$$
From this remark and Lemma \ref{lem:2unco}(ii) (where we use that Assumption \ref{ass:lip} implies $f\in\mathcal{F}_{2}(\bar L/2)$) we obtain 
\begin{eqnarray}
\label{eq:basic}
 &&\sum_{t=1}^T \mathbb{E} [ f(x_t) - f(x)]  \leq  \frac{1}{\alpha} \sum_{t=1}^T 
  \left( (\kappa_\beta L)^2 d^2 h_t^{2(\beta-1)}   + \frac{1}{t}\mathbb{E} [\|{\hat g}_t\|^2]\right) 
  \\ \nonumber
  && \qquad \leq  \frac{1}{\alpha} \sum_{t=1}^T 
  \left( (\kappa_\beta L)^2 d^2 h_t^{2(\beta-1)}   + \frac{1}{t}\Big[
  9\kappa\Big(G^2 d + \frac{{\bar L}^2d^2h_t^2}{8} \Big) + \frac{3\kappa d^2 \sigma^2}{2h_t^2}\Big] \right) 
  \\ \label{eq:lll}
  && \qquad \leq  \frac{d^2}{\alpha} \sum_{t=1}^T \Big[\Big\{ (\kappa_\beta  L)^2 h_t^{2(\beta-1)} + \frac{3}{2} \frac{\kappa\sigma^2}{h_t^2 t}\Big\} + \frac{9\kappa {\bar L}^2h_t^2}{8t}\Big]  +  \frac{9\kappa G^2}{\alpha} d (\log T + 1).
  \label{eq:onl}
  \end{eqnarray}
If $\sigma>0$ then our choice of
$h_t = \left(\frac{3\kappa\sigma^2}{2(\beta-1) (\kappa_\beta L)^2} \right)^{\frac{1}{2\beta}} t^{-\frac{1}{2\beta}}
$ is the minimizer of the main term (in curly brackets in  \eqref{eq:onl}). 
Plugging this $h_t$ in \eqref{eq:onl} and using the fact that $\sum_{t=1}^T t^{-1+1/\beta}\le \beta T^{1/\beta}$
for $\beta\ge 2$
we get \eqref{eq:bound0}.
Inequality \eqref{eq:bound1} follows from \eqref{eq:bound0} in view of the convexity of $f$. 
If $\sigma=0$ the stochastic variability term in \eqref{eq:onl} disappears and one can choose $h_t$ as small as desired, in particular, such that the sum in \eqref{eq:onl} is smaller than $\frac{\kappa G^2}{\alpha} d \log T$. This yields the bounds for $\sigma=0$.
\end{proof}
\fththeorem*
\begin{proof}
We have 
\begin{align*}
 \mathbb{E} \lvert \hat M - f(x^*) \rvert &\leq
\mathbb{E} \Big\lvert \frac{1}{T}\sum_{t=1}^T \xi''_t \Big\rvert +
\mathbb{E} \Big\lvert \frac{1}{T}\sum_{t=1}^T (f(x_t) - f(x^*)) \Big\rvert\\
& =
\mathbb{E} \Big\lvert \frac{1}{T}\sum_{t=1}^T \xi''_t \Big\rvert +
\frac{1}{T}\sum_{t=1}^T \mathbb{E}[f(x_t) - f(x^*)]\\
&\le 
\frac{\sigma}{T^{\frac{1}{2}}}
+
\frac{1}{T}\sum_{t=1}^T \mathbb{E}[f(x_t) - f(x^*)]
\end{align*}
and the theorem follows by using \eqref{eq:bound0}.
\end{proof}
\fiftheorem*
\begin{proof}
We start as in the proof of Theorem \ref{thm:1} to get \eqref{eq:main}. Then, using the strong convexity of $f$ and the fact that $x^*$ is the minimizer of $f$ we get analogously to \eqref{eq:decouple} that
$$
d h_t^{\beta-1} \|x_t - x^*\| \leq \frac{1}{2} \Big( \frac{2\kappa_\beta  L}{\alpha} d^2 h_t^{2(\beta-1)}   + \frac{\alpha}{2\kappa_\beta L} \|x_t-x^*\|^2 \Big)\leq  \frac{\kappa_\beta L}{\alpha} d^2h_t^{2(\beta-1)} + \frac{f(x_t) - f(x^*)}{2\kappa_\beta L  }. 
$$
Combining the last display and \eqref{eq:main}, using Lemma \ref{lem:2unco} and 
letting 
$r_t = \mathbb{E} [\|x_t-x^*\|^2]$ we get 
\begin{equation}
\label{eq:B}
\mathbb{E} [ f(x_t) - f(x^*)]  \leq \frac{r_t - r_{t+1}}{\eta_t} - \alpha r_t + 2(\kappa_\beta L)^2 \frac{d^2}{\alpha} h_t^{2(\beta-1)} +  \kappa\eta_t
\left[9{\bar L}^2  \left(d r_t {+} \frac{d^2h_t^2}{8} \right)+ \frac{3d^2 \sigma^2}{2h_t^2} \right].
\end{equation}
For $t=1,\dots,T_0$, since $h_t = T^{-\frac{1}{2\beta}}$ and $\eta_t = (\alpha T)^{-1}$ we have the following consequence of \eqref{eq:B}
\begin{equation}
\label{eq:123}
r_{t+1} \leq r_t\left(1-\frac{1}{T} + \frac{9 \kappa  {\bar L}^2}{(\alpha T)^2} d\right) + b_T\leq r_t\left(1+ \frac{9 \kappa  {\bar L}^2}{(\alpha T)^2} d\right) + b_T
\end{equation}
where
\begin{eqnarray}
\nonumber
b_T & =  &\frac{d^2}{\alpha^2 T} \left(2 (\kappa_\beta  L)^2 T^{-\frac{\beta-1}{\beta}} + 
\frac{9}{8} \kappa {\bar L}^2 T^{-\frac{\beta+1}{\beta}} + \frac{3}{2} \kappa \sigma^2 T^{-\frac{\beta-1}{\beta}} \right) \leq \\
& \leq &
\frac{d^2}{\alpha^2 T} \left(
2(\kappa_\beta  L)^2 + \frac{9}{8} \kappa {\bar L}^2
+ \frac{3}{2} \kappa \sigma^2
\right) T^{-\frac{\beta-1}{\beta}} .
\end{eqnarray}
Letting $C_3 = 9 \kappa {\bar L}^2$, inequality \eqref{eq:123} is of the form
$r_{t+1} \leq r_t q + b_T$, with $q=(1+ \frac{C_3d}{(\alpha T)^2})$. Then  
$$
r_{T_0+1} \leq r_1 q^{T_0} + b_T \sum_{j=1}^{T_0-1} q^j \leq r_1 q^{T_0} + b_T \frac{q^{T_0}}{q-1} \leq \left(r_1 + \frac{(\alpha T)^2}{C_3 d} b_T\right) q^{T_0}.
$$
Now, assuming
\begin{equation}
\label{eq:zzz}
T_0 = \left\lfloor{\frac{4C_3 d}{\alpha^2}}\right\rfloor
\end{equation}
 we obtain 
\begin{eqnarray*}
q^{T_0} & = & \exp\left[T_0 \log \left(1{+} \frac{C_3d}{(\alpha T)^2}\right) \right] \\
& \leq &
\exp\left[\frac{4C_3 d}{\alpha^2} \log \left(1{+} \frac{C_3d}{(\alpha T)^2}\right) \right]\\
&  \leq & \exp\left(\frac{4C_3^2 d^2}{\alpha^4 T^2}\right)\leq \exp\left(\frac{4C_3^2}{C_*^2}\right)=: C_4
\end{eqnarray*}
where in the last inequality we have used the assumption that, for $C_*>0$ large enough,
\begin{equation}
\label{eq:uuu}
\alpha > \sqrt{\frac{C_* d}{T}}.
\end{equation}
As we shall see, this also guarantees that $T_0 < T$. In conclusion, we obtain
\begin{eqnarray}\nonumber
r_{T_0+1} & \leq & C_4\left(r_1 + \frac{(\alpha T)^2}{C_3 d} b_T\right)  \\ \nonumber
& \leq & C_4 \left(r_1 +
\frac{(\alpha T)^2}{C_3 d} \frac{d^2}{\alpha^2 T} \left(
2 (\kappa_\beta L)^2 + \frac{9}{8} \kappa {\bar L}^2
+ \frac{3}{2} \kappa \sigma^2
\right) T^{-\frac{\beta-1}{\beta}} \right)\\
& = &C_4 \left(r_1 +
\frac{d}{C_3} \left(
2 (\kappa_\beta L)^2 + \frac{9}{8} \kappa {\bar L}^2
+ \frac{3}{2} \kappa \sigma^2
\right) T^{\frac{1}{\beta}} \right).
\label{eq:2222}
\end{eqnarray}
We now go back to inequality \eqref{eq:B}. 
Recalling the definition of ${\bar x}_{T_0,T}$ and the fact that  $h_t = t^{-\frac{1}{2\beta}}$
and $\eta_t = \frac{2}{\alpha t}$ for $t \in \{T_0+1,\dots T\}$, we deduce from \eqref{eq:B} that
\begin{eqnarray}\nonumber
(T-T_0) \mathbb{E} [ f({\bar x}_{T_0,T}) - f(x^*)] & \leq &\sum_{t=T_0+1}^T (r_t - r_{t+1}) \frac{\alpha t}{2} - \alpha r_t  +  18 \kappa \frac{{\bar L}^2}{\alpha t} d r_t \\
&&+ \frac{d^2}{\alpha}\sum_{t=T_0+1}^T \left(
2(\kappa_\beta L)^2 t^{-\frac{\beta-1}{\beta}}+  \frac{9}{4}\kappa {\bar L}^2 t ^{-\frac{\beta+1}{\beta}} + 
3 \kappa\sigma^2 t^{-\frac{\beta-1}{\beta}}\right).
\nonumber
\end{eqnarray}
Since $9 \kappa {\bar L}^2 = C_3$ condition \eqref{eq:zzz} implies that $\frac{18 \kappa {\bar L}^2 }{\alpha t } d \leq \frac{\alpha}{2}$ for $ t \geq T_0+1$. Thus
$$
(T-T_0) \mathbb{E} [ f({\bar x}_{T_0,T}) - f(x^*)] \leq \frac{\alpha}{2} \sum_{t=T_0+1}^T \Big[(r_t - r_{t+1}) t - r_t\Big] + U_T,
$$
where
$$
U_T =  \frac{d^2}{\alpha} \left(
2(\kappa_\beta L)^2 +  \frac{9}{4} \kappa{\bar L}^2 + 
3 \kappa  \sigma^2 \right) \sum_{t=T_0}^T  t^{-\frac{\beta-1}{\beta}} \leq \frac{d^2}{\alpha} \left(
2(\kappa_\beta L)^2 +  \frac{9}{4} \kappa{\bar L}^2 + 
3 \kappa  \sigma^2 \right) \beta T^{\frac{1}{\beta}}.
$$
On the other hand 
$$
 \sum_{t=T_0+1}^T 
\Big[(r_t - r_{t+1}) t - r_t\Big]
 \leq  r_{T_0+1} (T_0+1 - 1) +
 \sum_{t=T_0+2}^T r_t( t - (t-1)-1) = { T_0} r_{T_0+1}.
$$
Using inequality \eqref{eq:2222} and condition \eqref{eq:zzz} we get 
\begin{eqnarray*}
\frac{\alpha T_0}{2} r_{T_0+1} & \leq & \frac{2 C_3 C_4 d}{\alpha} \left(r_1 +
\frac{d}{C_3} \left(
2 (\kappa_\beta L)^2 + \frac{9}{8} \kappa {\bar L}^2
+ \frac{3}{2} \kappa \sigma^2
\right) T^{\frac{1}{\beta}} \right) \\
& = & 2 C_4 \left(9 \kappa {\bar L}^2 \frac{d}{\alpha}r_1 +
\frac{d^2}{\alpha} \left(
2 (\kappa_\beta L)^2 + \frac{9}{8} \kappa {\bar L}^2
+ \frac{3}{2} \kappa \sigma^2
\right) T^{\frac{1}{\beta}} \right).
\end{eqnarray*}
These bounds imply
$$
(T-T_0) \mathbb{E} [ f({\bar x}_{T_0,T}) - f(x^*)] \leq 18 C_4 \kappa {\bar L}^2 \frac{d}{\alpha}r_1 
+ (2C_4 +\beta)
\frac{d^2}{\alpha} \left(
2 (\kappa_\beta L)^2 + \frac{9}{4} \kappa {\bar L}^2
+ {3} \kappa \sigma^2
\right) T^{\frac{1}{\beta}}.
$$
Since $C_* > 8C_3= 72 \kappa {\bar L}^2$ it follows from \eqref{eq:zzz} and \eqref{eq:uuu} that $T\geq 2T_0$. Thus 
$$
 \mathbb{E} [ f({\bar x}_{T_0,T}) - f(x^*)]  \leq 
 36 C_4 \kappa {\bar L}^2 \frac{d}{\alpha T} r_1 
 + \Big(4C_4+2\beta\Big)
\frac{d^2}{\alpha} \left(
2 (\kappa_\beta  L)^2 + \frac{9}{4} \kappa {\bar L}^2
+ 3 \kappa \sigma^2
\right) 
T^{-\frac{\beta-1}{\beta}}.
$$
\end{proof}

\eiththeorem*
\begin{proof}
Fix $x\in \com$. Due to the $\alpha$-strong convexity  of $\hat f_t$ (cf. Lemma \ref{lem6.4}) we have 
$$
\hat f_t(x_t) - \hat f_t(x) \le  \langle\nabla \hat f_t(x_t), x_t-x\rangle - \frac{\alpha}{2} \norm{x_t- x}^{2}.
$$
 Using \eqref{lem6.4_1} and Lemma \ref{lem6.2}(ii) we obtain
$$
f(x_t) - f(x) \le L h_t^2 + \langle\nabla \hat f_t(x_t), x_t-x\rangle - \frac{\alpha}{2} \norm{x_t- x}^{2}.
$$
Using this property and exploiting inequality~\eqref{eq:GD} we find, with an argument similar to the proof of Theorem \ref{thm:1}, that
\begin{align}\label{ef5a:th6.3}
\forall x\in \Theta: \qquad 	\mathbb{E} \big[f(x_t) - f(x)\big]  \le L h_t^2   + \frac{r_t-r_{t+1}}{2\eta_t} - \frac{\alpha}{2} r_t + \frac{\eta_t}{2} \mathbb{E} [\|{\hat g}_t \|^2 ].
\end{align}
By assumption, $\eta_t = \frac{1}{\alpha t}$. Summing up from $t=1$ to $T$ and reasoning again analogously to the proof of Theorem~\ref{thm:1} we obtain  
\begin{align}\label{ef6:th6.3}
\forall x\in \Theta: \qquad \mathbb{E} \sum_{t=1}^{T} \big(f(x_t) - f(x)\big)  
&\le \sum_{t=1}^{T} \Big( L h_t^2 +  \frac{1}{2\alpha t} \mathbb{E} [\|{\hat g}_t \|^2 ]\Big). 
\end{align}
Now, inspection of the proof of Lemma \ref{lem:2unco} shows that it remains valid with $\kappa=1$ when $K(\cdot)\equiv 1$ in Algorithm \ref{algo}. This yields 
$$
\mathbb{E} [\|{\hat g}_t \|^2 ] \leq 
9  \left(G^2 d + \frac{L^2d^2h_t^2}{2} \right) + \frac{3 d^2 \sigma^2}{2h_t^2}. 
$$
Thus,
\begin{align}\label{ef7:th6.3}
\forall x\in \Theta: \qquad \mathbb{E} \sum_{t=1}^{T} \big(f(x_t) - f(x)\big)  
&\le \sum_{t=1}^{T} \Big[ \Big( L +  \frac{9L^2d^2}{4\alpha t} \Big)h_t^2 
+ \frac{3 d^2 \sigma^2}{4h_t^2\alpha t}+
 \frac{9G^2d}{2\alpha t} \Big]. 
\end{align}
The chosen value $h_t= \left(\frac{3d^2\sigma^2}{4L\alpha t+9L^2d^2}\right)^{1/4}$ minimizes the r.h.s. and yields
\begin{align*}
\label{ef8:th6.3}
\forall x\in \Theta: \ \ \mathbb{E} \sum_{t=1}^{T} \big(f(x_t) - f(x)\big)  
&\le  \frac{3}{2}\sum_{t=1}^{T} 
\frac{d^2\sigma^2}{\alpha t}
\left(\frac{4L\alpha t+9L^2d^2}{3d^2\sigma^2}\right)^{1/2} + \frac{9G^2}{2}\frac{d}{\alpha} (1+\log T)
\nonumber \\
&\le 
\sum_{t=1}^{T} \sqrt{3} \Big[
\frac{d\sigma\sqrt{L}}{\sqrt{\alpha t}} + \frac{3Ld^2\sigma}{2\alpha t}  \Big]+ 9G^2 \frac{d}{\alpha} (1+\log T) \nonumber \\
& \le  2 \sqrt{3 L} \sigma 
\frac{d}{\sqrt{\alpha}} \sqrt{T} +  \Big(\frac{3\sqrt{3}}{2} \sigma L  +  \frac{9G^2}{d}  \Big) \frac{d^2}{\alpha} (1+\log T).
\end{align*}
As $1+\log T\le ((\log 2)^{-1}+1)\log T$ for any $T\ge 2$, we obtain  \eqref{eq1:th6.3}. On the other hand, we have the straightforward bound
\begin{equation}
\label{ef8:th6.3}
\forall x\in \Theta: \qquad \mathbb{E} \sum_{t=1}^{T} \big(f(x_t) - f(x)\big)  \le GBT.
\end{equation}
The remaining part of the proof follows the same lines as in Theorem \ref{thm:1}.  
\end{proof}

\lowerB*
\begin{proof}
We use the fact that $\sup_{f\in\mathcal{F}'_{\alpha,\beta}}$ is bigger than the maximum over a finite family of functions in $\mathcal{F}'_{\alpha,\beta}$. We choose this finite family in a way that its members cannot be distinguished from each other with positive probability but are separated enough from each other to guarantee that the maximal optimization error for this family is of the order of the desired lower bound.  

We first assume that $\alpha\ge T^{-1/2+1/\beta}$. 

Let $\eta_{0} : \mathbb{R} \to \mathbb{R}$ be an infinitely many times differentiable function such that
\begin{equation*}
   \eta_{0}(x) = \begin{cases}
      =1 & \text{if $|x|\leq 1/4$},\\
      \in (0,1) & \text{if $1/4 < |x| < 1$},\\
      =0 & \text{if $|x| \geq 1$}.
    \end{cases} 
\end{equation*}
Set $\eta(x) = \int_{-\infty}^{x} \eta_{0}(\tau)d\tau$. Let $\Omega = \big\{-1,1\big\}^{d}$ be the set of binary sequences of length $d$. 
Consider the finite set of functions $f_{\omega}: \mathbb{R}^{d}\to \mathbb{R}, \omega\in\Omega$, defined as follows:
\[
f_{\omega}(u) = \alpha(1+\delta) \norm{u}^{2}/2 + \sum_{i=1}^{d}\omega_{i}rh^{\beta}\eta(u_{i}h^{-1}),
\qquad u=(u_1,\dots,u_d),
\]
where $\omega_i\in \{-1,1\}$, 
 $h =\min\big((\alpha^2/d)^{\frac{1}{2(\beta-1)}}, T^{-\frac{1}{2\beta}}\big)$ and $r>0, \delta >0$ are fixed numbers that will be chosen small enough. 
 
 Let us prove that $f_{\omega}\in\mathcal{F}'_{\alpha,\beta}$
 for $r>0$ and $\delta >0$ small enough. It is straightforward to check that if $r$ is small enough the functions $f_{\omega}$ are $\alpha$-strongly convex
 and belong to $\mathcal{F}_{\beta}(L)$.
 
 Next, the components of the gradient $\nabla f_{\omega}$ have the form
 $$( \nabla f_{\omega}(u))_{i} = \alpha (1 + \delta)u_{i}+\omega_{i}r h^{\beta-1}\eta_0(u_{i}h^{-1}).$$
 Thus, 
 $$\norm{\nabla f_{\omega} (u)}^2\le 2\alpha^2 (1 + \delta)^2\norm{u}^{2} + 2 r^2 \alpha^2
 $$
and the last  expression can be rendered smaller than $G^2$ uniformly in $u\in \com$ by the choice of $\delta$ and $r$ small enough since $G^2> 4\alpha^2 $. 

Finally, we check that the minimizers of functions $f_{\omega}$ belong to $\com$.  
Notice that we can choose $r$ small enough to have
$\alpha^{-1}(1+\delta)^{-1}rh^{\beta-2}<1/4$
and that under this condition 
 the equation $\nabla f_{\omega}(x) = 0$
has the solution $$x_{\omega}^{*} = (x^*(\omega_1), \dots, x^*(\omega_d)), $$
where $x^*(\omega_i)=-\omega_{i}\alpha^{-1}(1+\delta)^{-1}r h^{\beta-1}$.
Using the definition of $h$ we obtain
$$
\norm{x_{\omega}^{*}}\le d^{1/2}  \alpha^{-1}(1+\delta)^{-1}r h^{\beta-1} \le d^{1/2}  \alpha^{-1}(1+\delta)^{-1}r (\alpha^2/d)^{1/2}\le (1+\delta)^{-1}r <1
$$
for $r>0$ small enough, which means that $x_{\omega}^{*}$ belongs to the interior of $\com$.

Combining all the above remarks we conclude that the family of functions $\{f_{\omega}, \omega\in\Omega\}$ is a subset of $\mathcal{F}'_{\alpha,\beta}$ for $r>0$ and $\delta >0$ small enough.



Set for brevity $(z_i,y_i)_{i=1}^t=(z_1,y_{1},\dots,z_t, y_{t})$,  $(\boldsymbol{\zeta}_i)_{i=1}^t=(\boldsymbol{\zeta}_1,\dots,\boldsymbol{\zeta}_t)$. For any fixed $\omega\in\Omega$, we denote by $\mathbf{P}_{\omega,T}$ the probability measure corresponding to the joint distribution of $((z_i,y_i)_{i=1}^T,(\boldsymbol{\zeta}_i)_{i=1}^T)$ where $y_{t}=f_\omega (z_{t})+\xi_{t}$ with independent identically distributed $\xi_{t}$'s such that \eqref{distribution} holds, $\xi_t$ is independent of $(z_1,y_1,\dots, z_{t-1},y_{t-1},\boldsymbol{\zeta}_t)$ for each $t$, and $z_t$'s chosen by a sequential strategy in $\Pi_T$.
We have
$$
d\mathbf{P}_{\omega,T}((z_i,y_i)_{i=1}^T,(\boldsymbol{\zeta}_i)_{i=1}^T)
=dF\big(y_{1}-f_\omega(z_{1})\big)\prod_{t=2}^{T}dF\Big(y_{t}-f_\omega\big(\Phi_t((z_i,y_i)_{i=1}^{t-1},\boldsymbol{\zeta}_{t}\big)\Big) d\mathbb{P}_t(\boldsymbol{\zeta}_t),
$$
where $\mathbb{P}_t$ is the probability measure corresponding to the distribution of $\boldsymbol{\zeta}_t$.
Let $\mathbf{E}_{\omega,T}$ denote the expectation w.r.t. $\mathbf{P}_{\omega,T}$. Consider the statistic
$$\hat{\omega} \in \argmin_{\omega \in \Omega} \norm{z_{T}-x^{*}_{\omega}}.$$
Since
$\norm{x^{*}_{\hat{\omega}}-x^{*}_{{\omega}}}\leq \norm{z_{T}-x^{*}_{{\omega}}}+\norm{z_{T} - x^{*}_{\hat{\omega}}}\leq 2\norm{z_{T}-x^{*}_{\omega}}$
for all $\omega \in \Omega$ we obtain
\begin{align*}
    \mathbf{E}_{\omega,T}\big[\norm{z_{T}-x^{*}_{\omega}}^{2}\big] &\geq \frac{1}{4}\mathbf{E}_{\omega,T}\big[\norm{x_{\omega}^{*}-x_{\hat{\omega}}^{*}}^{2}\big]\\
    &=\alpha^{-2}r^{2} h^{2\beta-2}\mathbf{E}_{\omega,T}\rho(\hat{\omega},\omega),
\end{align*}
where $\rho(\hat{\omega},\omega)= \sum_{i=1}^{d}\mathbb{I}(\hat{\omega}_i\ne\omega_i)$ is the Hamming distance between $\hat{\omega}$ and $\omega$. Taking the maximum over $\Omega$ and then the minimum over all statistics $\hat{\omega}$ with values in $\Omega$ we obtain
$$\max_{\omega \in \Omega}\mathbf{E}_{\omega,T}\big[\norm{z_{T}-x^{*}_{\omega}}^{2}\big] \geq \alpha^{-2}r^{2} h^{2\beta-2}\inf_{\hat{\omega}}\max_{\omega \in \Omega}\mathbf{E}_{\omega} \rho(\hat{\omega},\omega).$$
 By \cite[Theorem 2.12]{Tsybakov09}, if for some $\gamma>0$ and all $\omega, \omega'\in \Omega$ such that $\rho(\omega, \omega')=1$ we have $KL(\mathbf{P}_{\omega,T} , \mathbf{P}_{\omega',T})\leq \gamma$, where  $KL(\cdot,\cdot)$ denotes the Kullback-Leibler divergence, then 
$$\inf_{\hat{\omega}}\max_{\omega \in \Omega}\mathbf{E}_{\omega, T} \rho(\hat{\omega},\omega)\geq \frac{d}{4} \exp(-\gamma).$$
Now for all $\omega, \omega'\in \Omega$ such that $\rho(\omega, \omega')=1$ we have
\begin{align*}
    KL(\mathbf{P}_{\omega,T} , \mathbf{P}_{\omega',T}) &= \int \log\Big(\frac{d\mathbf{P}_{\omega,T}}{d\mathbf{P}_{\omega',T}}\Big)d\mathbf{P}_{\omega,T}
    \\
    &=\int \bigg[\log\Big(\frac{dF(y_{1}-f_{\omega}(z_{1}))}{dF(y_{1}-f_{\omega'}(z_{1}))}\Big)+
    \\
    &~~~~~~~~~+\sum_{t=2}^{T}\log\bigg(\frac{dF(y_{t}-f_{\omega}\big(\Phi_t((z_i,y_i)_{i=1}^{t-1},\boldsymbol{\zeta}_{t})\big))}{dF(y_{t}-f_{\omega'}\big(\Phi_t((z_i,y_i)_{i=1}^{t-1},\boldsymbol{\zeta}_{t})\big))}\bigg)\bigg]
    \\
    &~~~~~~~~~~~~~~~~~~~~~dF\big(y_{1}-f_{\omega}(z_{1})\big)\prod_{t=2}^{T}dF\Big(y_{t}-f_{\omega}\big(\Phi_t((z_i,y_i)_{i=1}^{t-1},\boldsymbol{\zeta}_{t})\big)\Big)d\mathbb{P}_t(\boldsymbol{\zeta}_t)
    \\
    &\leq TI_{0}\max_{u\in \mathbb{R}}|f_{\omega}(u)-f_{\omega'}(u)|^{2}=I_{0}r^{2}\eta^{2}(1),
\end{align*}
where the last inequality is granted if $r<v_{0}/\eta(1)$ due to \eqref{distribution}. Assuming in addition that $r$ satisfies  $r^{2}\leq (\log2)/\big(I_{0}\eta^{2}(1)\big)$ we obtain $KL(\mathbf{P}_{\omega,T} , \mathbf{P}_{\omega',T})\leq \log 2$. 
Therefore, we have proved that if $\alpha\ge T^{-1/2+1/\beta}$ then there exist $r>0$ and $\delta >0$ small enough such that  
\begin{equation}\label{eq3:lb}
 \max_{\omega \in \Omega}\mathbf{E}_{\omega,T}\big[\norm{z_{T}-x^{*}_{\omega}}^{2}\big]\geq \frac{1}{8} d\alpha^{-2}r^{2} h^{2\beta-2} = \frac{r^{2}}{8}\min\Big(1, 
\,\frac{d}{\alpha^{2}}T^{-\frac{\beta-1}{\beta}}
\Big).   
\end{equation}
This implies \eqref{eq2:lb} for $\alpha\ge T^{-1/2+1/\beta}$. In particular, if $\alpha=\alpha_0:= T^{-1/2+1/\beta}$ the bound \eqref{eq3:lb} is of the order $\min\Big(1, {d}{T^{-\frac{1}{\beta}}}
\Big)$. Then for $0<\alpha<\alpha_0$ we also have the bound of this order since the classes $\mathcal{F}'_{\alpha,\beta}$ are nested: $\mathcal{F}'_{\alpha_0,\beta}\subset \mathcal{F}'_{\alpha,\beta}$. This completes the proof of \eqref{eq2:lb}. 

We now prove \eqref{eq1:lb}. From \eqref{eq3:lb} and $\alpha$-strong convexity of $f$ we get that, for $\alpha\ge T^{-1/2+1/\beta}$,
\begin{equation}\label{eq4:lb}
\max_{\omega \in \Omega}\mathbf{E}_{\omega,T}
\big[f(z_T)-f(x_\omega^*)\big]\geq
\frac{r^{2}}{16}\min\Big(\alpha, 
\,\frac{d}{\alpha}T^{-\frac{\beta-1}{\beta}}
\Big).  
\end{equation}
This implies \eqref{eq1:lb} in the zone $\alpha\ge T^{-1/2+1/\beta}$ since for such $\alpha$ we have
$$
\min\Big(\alpha, 
\,\frac{d}{\alpha}T^{-\frac{\beta-1}{\beta}}
\Big)
=
\min\Big(\max(\alpha, T^{-1/2+1/\beta}), \frac{d}{\sqrt{T}}, \,\frac{d}{\alpha}T^{-\frac{\beta-1}{\beta}}\Big).
$$
On the other hand, 
$$
\min\Big(\alpha_0, 
\,\frac{d}{\alpha_0}T^{-\frac{\beta-1}{\beta}}
\Big)
=
\min\Big(T^{-1/2+1/\beta}, \frac{d}{\sqrt{T}} \Big),
$$
and the same lower bound holds for $0<\alpha<\alpha_0$  by the nestedness argument that we used to prove \eqref{eq2:lb} in the zone $0<\alpha<\alpha_0$. Thus, \eqref{eq1:lb} follows.

\end{proof}

\section{Comments on \cite{BP2016}}
\label{app:D}
In this section we comment on issues with some claims
in the paper of Bach and Perchet \cite{BP2016}, which presents a number of valuable results and provides
a motivation for our work. We wish to clarify such issues for
the sake of understanding, as otherwise a comparison to the results
presented here would be misleading.

Bach and Perchet \cite{BP2016} introduce Algorithm \ref{algo} in the current form and provide upper bounds for its optimisation error and online regret when  
 $f\in \mathcal{F}_{\beta}(L)$ with integer $\beta$. The setting where $f$ is strongly convex is considered in Propositions 4,6-8 and 9 of that paper. Propositions 4, 6,9 give the rates decaying in $T$ not faster than $T^{-\frac{\beta-1}{\beta+1}}$, which is slower than the optimal rate $T^{-\frac{\beta-1}{\beta}}$.  Proposition 8 dealing with asymptotic results is problematic. It is stated as bounds on  $\norm{x_N-x^*}$ but the authors presumably mean bounds on $\mathbb{E}\norm{x_N-x^*}^2$. 
	A dependence of the bound on the initial value of the algorithm is missing in the part of Proposition 8 entitled "unconstrained optimization of strongly convex mappings". This remark also concerns Proposition 7.


\section{Additional results}
\label{app:C}

In this appendix, we provide refined versions of 
Theorems \ref{thm:1} and \ref{th6.3}.  First we state a non-asymptotic version of Chung's lemma \cite[Lemma 1]{chung}. It allows us to obtain in Theorem \ref{th:0-bis}  upper bounds for $\mathbb{E}\{\norm{x_{t}-x^{*}}^{2}\}$, where $x_{t}$ is generated by a constrained version of Algorithm $\ref{algo}$ (i.e., with compact $\com$) under the assumptions of Theorems \ref{thm:1} and \ref{th6.3}. By using this result and considering averaging from $\floor{T/2}+1$ to $T$ rather than from 1 to $T$, in Theorems \ref{thm:1-bis} and \ref{th6.3-bis} we provide finer upper bounds for the optimization error than in Theorems \ref{thm:1} and \ref{th6.3}. The refinement consists in the fact that we get rid of  the logarithmic factors appearing in (\ref{eq:bound1}) and (\ref{eq2:th6.3}). Finally, in  Theorem \ref{th6.3-bisx} we show that the term $\frac{d^2}{\alpha} \log T$ in the bound on the cumulative regret in Theorem \ref{th6.3} can be improved to $\frac{d}{\alpha} \log T$ under a slightly more restrictive assumption (we assume that the norm $\|\nabla f\|$ is uniformly bounded by $G$ on a large enough Euclidean neighborhood of $\com$ rather than only on $\com$).
\begin{lemma}
\label{lemma:bis}
    Let $\{b_{t}\}$ be a sequence of real numbers such that for all integers $t \geq 2$, 
    \begin{equation}
    \label{ch:0}
        b_{t+1} < \left(1 - \frac{1}{t} \right)b_{t} + \sum_{i=1}^{N}\frac{a_{i}}{t^{p_{i}+1}},
    \end{equation}
where $0<p_{i}<1$ and $a_{i}\geq0$ for $1 \leq i \leq N $. Then for $t \geq 2$ we have
\begin{equation}
\label{ch:res}
        b_{t} < \frac{2 b_{2}}{t}+\sum_{i=1}^{N}\frac{a_{i}}{(1-p_{i})t^{p_{i}}}.
    \end{equation}
\end{lemma}
\begin{proof}
For any fixed $t > 0$ the convexity of the mapping $u\mapsto g(u)=(t+u)^{-p}$ implies that  $g(1)-g(0)\ge g'(0)$, i.e.,  
\begin{eqnarray*}
\label{ch:1}
\frac{1}{t^{p}}-\frac{1}{(t+1)^{p}}
\leq \frac{p}{t^{p+1}}.
\end{eqnarray*}
Thus,
\begin{eqnarray}
\label{ch:2}
\frac{a_i}{t^{p+1}}\leq \frac{a_i}{1-p}\left(\frac{1}{(t+1)^{p}} - \Big(1-\frac{1}{t}\Big)\frac{1}{t^{p}} \right).
\end{eqnarray}
Using (\ref{ch:0}), and (\ref{ch:2}) and rearranging terms we get 
\begin{eqnarray}
\nonumber
b_{t+1}-\sum_{i=1}^{N}\frac{a_{i}}{(1-p_{i})(t+1)^{p_{i}}}\le\Big(1-\frac{1}{t}\Big)\left[b_{t} - \sum_{i=1}^{N}\frac{a_{i}}{(1-p_{i})t^{p_{i}}}\right].
\end{eqnarray}
Letting $\tau_{t} = b_{t} - \sum_{i=1}^{N}\frac{a_{i}}{(1-p_{i})t^{p_{i}}}$ we have $\tau_{t+1}\leq (1-\frac{1}{t}) \tau_{t}$. Now, if $\tau_{2} \leq 0$ then $\tau_{t} \leq 0$ for any $t \geq 2$ and thus
(\ref{ch:res}) holds. Otherwise, if $\tau_{2}>0$ then for $t\ge 3$ we have
$$\tau_{t}\leq \tau_{2}\prod_{i={2}}^{t-1}\Big(1-\frac{1}{i}\Big)\leq \frac{2\tau_{2}}{t}\le \frac{2b_{2}}{t},$$
where we have used the inequalities $\sum_{i=2}^{t-1}\log\Big(1-\frac{1}{i}\Big)\le - \sum_{i=2}^{t-1} \frac{1}{i} \le -\log(t-1) \le \log(2/t)$.
Thus, (\ref{ch:res}) holds in this case as well.
\end{proof}
\begin{theorem}
\label{th:0-bis}
Let $f\in {\cal F} _{\alpha, \beta}(L)$ with $\beta \geq 2$, $\alpha,L>0$, $\sigma>0$, and 
	  let Assumption \ref{ass1} hold. 
	  Consider Algorithm \ref{algo} where $\com$ is a convex compact subset  of $\mathbb{R}^d$ and assume that $\max_{x\in\com}\|\nabla f(x)\|\le G $.  
	  
	  (i) If Assumption \ref{ass:lip} holds, $h_t = \left(\frac{3\kappa\sigma^2}{2(\beta-1)(\kappa_\beta L)^2} \right)^{\frac{1}{2\beta}} t^{-\frac{1}{2\beta}}$ and  $\eta_t=\frac{2}{\alpha t}$ then for $t\geq 1$ we have 
	  \begin{equation}\label{eq1:th:0-bis}
	\mathbb{E}\big[\norm{x_{t}-x^{*}}^{2}\big] < \frac{2G^{2}}{\alpha^{2}t}+ A_5\frac{d^{2}}{\alpha^{2}}t^{-\frac{\beta-1}{\beta}}
	\end{equation}
	where $x^{*} = \argmin_{x \in \Theta}f(x)$ and $A_5>0$ is a constant that does not depend on $d,\alpha,t$.
	
	(ii) If $\beta =2$,  $h_t= \left(\frac{3d^2\sigma^2}{4L\alpha t+9L^2d^2}\right)^{1/4}$ and $\eta_t=\frac{1}{\alpha t}$ then for $t \geq 1$ we have that
	\begin{equation}\label{eq2:th:0-bis}
	\mathbb{E}\big[\norm{x_{t}-x^{*}}^{2}\big] < \frac{2G^{2}}{\alpha^{2}t}+A_{6}\frac{d}{\alpha^{\frac{3}{2}}t^{\frac{1}{2}}}+A_{7}\frac{d^{2}}{\alpha^{2}t},
	\end{equation}
	where $A_{6},A_{7}>0$  are constants that do not depend on $d,\alpha,t$.
\end{theorem}
\begin{proof}
Let $r_{t}=\mathbb{E}\norm{x_{t}-x^{*}}^{2}$. To prove the theorem, we will show that under the assumptions of the theorem $\{r_{t}\}$ satisfies (\ref{ch:0}) with suitable $a_i$ and $p_i$, and then use Lemma \ref{lemma:bis}.

We start by noticing that, in view of the $\alpha$-strong convexity of $f$ and the fact that $f$ is Lipschitz continuous with constant $G$ in $\Theta$ for any $t\geq1$ we have
\begin{eqnarray}
\label{eq:xx0}
\norm{x_{t}-x^{*}}^{2}\leq \frac{G^{2}}{\alpha^{2}}.
\end{eqnarray}
Thus, \eqref{eq1:th:0-bis} and \eqref{eq2:th:0-bis} hold for $t=1$ and it suffices to prove the theorem for $t\ge 2$.
The definition of Algorithm \ref{algo} gives that, for $t\geq 1$,
$$\norm{x_{t+1}-x^{*}}^{2}\leq \norm{x_{t}-x^{*}}^{2}-2\eta_{t}\langle \hat g_t , x_{t}-x^{*} \rangle + \eta_{t}^{2} \norm{\hat g_{t}}^{2}.$$
Taking conditional expectation of both sides of this inequality given $x_{t}$ we obtain
\begin{eqnarray}
\nonumber
\mathbb{E}[\norm{x_{t+1}-x^{*}}^{2}|x_{t}]&\leq& \norm{x_{t}-x^{*}}^{2}-2\eta_{t}\langle \mathbb{E}[\hat g_t|x_{t}] , x_{t}-x^{*} \rangle + \eta_{t}^{2} \mathbb{E}[\norm{\hat g_{t}}^{2}|x_{t}].
\end{eqnarray}
Using this inequality and Lemmas \ref{lem:1} and \ref{lem:2unco}(ii) we find
\begin{eqnarray}
\nonumber
\mathbb{E}[\norm{x_{t+1}-x^{*}}^{2}|x_{t}]&\leq& \norm{x_{t}-x^{*}}^{2} -2\eta_{t} \alpha\norm{x_{t}-x^{*}}^{2} + 2\eta_{t}\kappa_\beta L d h_t^{\beta-1} \|x_t - x^{*}\|+\\
\label{eq:bis1} &\quad \quad& +\eta_{t}^{2}\left[\left( 9\kappa  \left(G^2 d + \frac{{L}^2d^2h_t^2}{2} \right) + \frac{3\kappa d^2 \sigma^2}{2h_t^2}\right)\right].
\end{eqnarray}
On the other hand, for $\lambda>0$, we have
\begin{eqnarray}
\label{eq:xx}
dh_{t}^{\beta-1}\norm{x_{t}-x^{*}}\leq \frac{1}{2}\left(\frac{\kappa_\beta L}{\alpha \lambda} d^2 h_t^{2(\beta-1)}   + \frac{\alpha \lambda}{\kappa_\beta L} \|x_t-x^*\|^2 \right).
\end{eqnarray}
Combining (\ref{eq:xx}) and (\ref{eq:bis1}) we get
\begin{equation}
\label{eq:xx2}
\begin{array}{lcc}
\mathbb{E}[\norm{x_{t+1}-x^{*}}^{2}|x_{t}]&\leq&(1-(2-\lambda)\eta_{t}\alpha) \norm{x_{t}-x^{*}}^{2} +  \frac{(\kappa_\beta L)^{2}}{\alpha \lambda}\eta_{t} d^2 h_t^{2(\beta-1)}+ \\ \\&\quad \quad& +\eta_{t}^{2}\left[\left( 9\kappa  \left(G^2 d + \frac{{L}^2d^2h_t^2}{2} \right) + \frac{3\kappa d^2 \sigma^2}{2h_t^2}\right)\right].
\end{array}
\end{equation}

Substituting $h_{t}=\left(\frac{3\kappa\sigma^2}{2(\beta-1)(\kappa_\beta L)^2} \right)^{\frac{1}{2\beta}} t^{-\frac{1}{2\beta}}$, $\eta_{t}=\frac{2}{\alpha t}$, $\lambda =\frac{3}{2}$ in (\ref{eq:xx2}), and taking the expectation over $x_{t}$ we obtain
\begin{eqnarray*}
\nonumber
r_{t+1}&\leq&\Big(1-\frac{1}{t}\Big) r_{t} +  \frac{4(\kappa_\beta L)^{2}}{3\alpha^{2}}d^2 \left(\frac{3\kappa\sigma^2}{2(\beta-1)(\kappa_\beta L)^2} \right)^{\frac{\beta-1}{\beta}}t^{-\frac{2\beta-1}{\beta}}+\\
&\quad \quad&+\frac{18\kappa L^{2}d^{2}}{\alpha^{2}}\left(\frac{3\kappa\sigma^2}{2(\beta-1)(\kappa_\beta L)^2} \right)^{\frac{1}{\beta}}t^{-\frac{2\beta+1}{\beta}}+\frac{36\kappa}{\alpha^{2}t^{2}}G^{2}d+\\
&\quad \quad&+\frac{6\kappa d^{2}\sigma^{2}}{\alpha^{2}}\left(\frac{3\kappa\sigma^2}{2(\beta-1)(\kappa_\beta L)^2} \right)^{-\frac{1}{\beta}}t^{-\frac{2\beta-1}{\beta}}.
\end{eqnarray*}
Thus, we have 
\begin{eqnarray*}
\label{eq:xx3}
r_{t+1}&<&\Big(1-\frac{1}{t}\Big)r_{t}+C\frac{d^{2}}{\alpha^{2}}t^{-\frac{2\beta-1}{\beta}},
\end{eqnarray*}
where 
\begin{eqnarray}
\nonumber
C &=&\frac{4(\kappa_\beta L)^{2}}{3} \left(\frac{3\kappa\sigma^2}{2(\beta-1)(\kappa_\beta L)^2} \right)^{\frac{\beta-1}{\beta}} +
18\kappa L^{2}\left(\frac{3\kappa\sigma^2}{2(\beta-1)(\kappa_\beta L)^2} \right)^{\frac{1}{\beta}}+
\\ \nonumber
&&+\frac{36\kappa}{d}G^{2}+6\kappa \sigma^{2}\left(\frac{3\kappa\sigma^2}{2(\beta-1)(\kappa_\beta L)^2} \right)^{-\frac{1}{\beta}}.
\end{eqnarray}
This is a particular instance of (\ref{ch:0}). Therefore, we can apply Lemma \ref{lemma:bis}, which yields that, for all $t\geq 2$, 
\begin{eqnarray*}
\label{eq:xx34}
r_{t}<\frac{2G^{2}}{\alpha^{2}t}+\beta C\frac{d^{2}}{\alpha^{2}}t^{-\frac{\beta-1}{\beta}}.
\end{eqnarray*}
Thus, \eqref{eq1:th:0-bis} follows. 

We now prove \eqref{eq2:th:0-bis}. 
Since $\beta = 2$, using Lemmas \ref{lem:1}, \ref{lem:2unco}(ii), and \ref{lem6.4} we obtain 
\[
\mathbb{E}[\norm{x_{t+1}-x^{*}}^{2}|x_{t}] \leq (1-\eta_{t}\alpha)\norm{x_{t}-x^{*}}^{2}+2\eta_{t}Lh_{t}^{2}  +\eta_{t}^{2}\left[\left( 9\left(G^2 d + \frac{{L}^2d^2h_t^2}{2} \right) + \frac{3 d^2 \sigma^2}{2h_t^2}\right)\right].
\]
Setting here $h_t= \left(\frac{3d^2\sigma^2}{4L\alpha t+9L^2d^2}\right)^{1/4} $, $\eta_t=\frac{1}{\alpha t}$, and taking the expectation over $x_{t}$ we get
\begin{eqnarray*}
\nonumber
r_{t+1} &\leq& \Big(1-\frac{1}{t}\Big)r_{t}+\left(\frac{(4L\alpha t+9{L}^2d^2)^{1/2}}{\alpha^{2}} \right)\frac{\sqrt{3}d\sigma}{t^{2}} + \frac{9G^{2}d}{\alpha^{2}t^{2}}
\\\label{eq:2xx}&\leq&\Big(1-\frac{1}{t}\Big)r_{t}+A_{6}'\frac{d}{\alpha^{\frac{3}{2}}t^{\frac{3}{2}}}+A_{7}'\frac{d^{2}}{\alpha^{2}t^{2}},
\end{eqnarray*}
where $A_{6}'=2\sqrt{3L}\sigma$ and $A_{7}'=3\sqrt{3L}\sigma+\frac{9{G}^{2}}{d}$.
Applying Lemma \ref{lemma:bis} for $t\geq 2$ we get
\begin{eqnarray}
\nonumber
r_{t} < \frac{2G^{2}}{\alpha^{2}t}+2A_{6}'\frac{d}{\alpha^{\frac{3}{2}}t^{\frac{1}{2}}}+2A_{7}'\frac{d^{2}}{\alpha^{2}t}.
\end{eqnarray}
\end{proof}
Consider the estimator 
\begin{eqnarray}
\label{est:s}
\hat{x}_{T} = \frac{1}{T-\floor*{{T}/{2}}}\sum_{t=\floor{{T}/{2}}+1}^{T}x_{t}.
\end{eqnarray}
The following two theorems provide bounds on the optimization error of this estimator. 
\begin{theorem}
\label{thm:1-bis}
Let $f\in {\cal F} _{\alpha, \beta}(L)$ with $\beta \geq 2$, $\alpha,L>0$, $\sigma>0$, and 
	  let Assumptions \ref{ass1} and \ref{ass:lip} hold. 
	  Consider Algorithm \ref{algo} where $\com$ is a convex compact subset  of $\mathbb{R}^d$ and assume that $\max_{x\in\com}\|\nabla f(x)\|\le G $.
If $h_t = \left(\frac{3\kappa\sigma^2}{2(\beta-1)(\kappa_\beta L)^2} \right)^{\frac{1}{2\beta}} t^{-\frac{1}{2\beta}}$ and $\eta_t=\frac{2}{\alpha t}$ then  the optimization error of the estimator  (\ref{est:s}) satisfies 
	 \begin{eqnarray}
\nonumber
 \mathbb{E} [ f(\hat{x}_{T}) - f(x^{*})]
  &\leq&  \min\left(GB, \frac{1}{\alpha}\bigg(d^2
 \Big(\frac{A_1'}{T^{\frac{\beta-1}{\beta}}} +\frac{A_2'}{T} \Big) +\frac{A_{3}'d}{T}\bigg)\right),
  \end{eqnarray}
  where $x^{*} = \argmin_{x \in \Theta}f(x)$. Here $A_{1}',A_{2}'$ and $A_{3}'$ are  positive constants that do not depend on $d,\alpha,T$, and $B$ is the Euclidean diameter of $\com$.
\end{theorem}
\begin{proof}
With the same steps as in the proof of Theorem \ref{thm:1} (see (\ref{eq:lll})) but taking now the sum over $t=\floor{{T}/{2}}+1,\dots,T$ rather than over $t=1,\dots,T$ we obtain
\begin{align*}
\nonumber
 \sum_{t=\floor{{T}/{2}}+1}^T \mathbb{E} [ f(x_t) - f(x^{*})]
  &\leq r_{\floor{{T}/{2}}+1}\frac{\floor{{T}/{2}}\alpha}{2}+ \frac{1}{\alpha} \sum_{t=\floor{{T}/{2}}+1}^T\Big( (\kappa_\beta L)^2 d^2 h_t^{2(\beta-1)} +\\\label{eq:ni}&\quad \quad  +\frac{1}{t}\Big[9\kappa
  \Big(G^2 d + \frac{{\bar L}^2d^2h_t^2}{8} \Big) + \frac{3\kappa d^2 \sigma^2}{2h_t^2}\Big] \Big)
  \\
  &\leq r_{\floor{{T}/{2}}+1}\frac{\floor{{T}/{2}}\alpha}{2}
  +\frac{9\kappa G^2 d }{\alpha} \sum_{t=\floor{{T}/{2}}+1}^T \frac{1}{t}
  \\ &\quad \quad
  + \frac{1}{\alpha} \sum_{t=1}^T \Big( (\kappa_\beta L)^2 d^2 h_t^{2(\beta-1)} + \frac{{\bar L}^2d^2h_t^2}{8t}
  + \frac{3\kappa d^2 \sigma^2}{2h_t^2 t}\Big).
\end{align*}
For the last sum here, we use 
exactly the same bound as in the proof of Theorem \ref{thm:1}. Moreover, it follows from Theorem \ref{th:0-bis} that 
$$r_{\floor*{{T}/{2}}+1} < \frac{4G^{2}}{\alpha^{2}T}+A_{5}'\frac{d^{2}}{\alpha^{2}}T^{-\frac{\beta-1}{\beta}},$$
where $A_{5}' = 2^{(\beta-1)/\beta}A_5$. 
Combining these remarks and using the fact that $\sum_{t=\floor*{{T}/{2}}+1}^{T} \frac{1}{t} \le \log(T/\floor*{{T}/{2}}) \le 2$ for all $T\ge 2$ (recall that we assume $T\ge 2$ throughout the paper), as well as the the convexity of $f$ we get
\begin{eqnarray}
\nonumber
  \mathbb{E} [ f(\hat{x}_{T}) - f(x^{*})]
  &\leq&  \frac{1}{\alpha}\Bigg(
d^2 \Big(\frac{A_1'}{T^{\frac{\beta-1}{\beta}}}+\frac{A_{2}'}{T} \Big) +\frac{A_{3}'d}{T}{\Bigg)},
  \end{eqnarray}
 where $A_{1}'={2A_{1}+\frac{A_{5}'}{2}}$, 
$A_2'=2\bar c \bar{L}^{2}(\sigma/L)^{\frac{2}{\beta}}$
with constant $\bar c$ as in Theorem \ref{thm:1} and 
 $A_{3}'=2G^2 (18 \kappa +1/d)$. 
On the other hand we have the straightforward bound
\begin{equation}
\nonumber
\mathbb{E} [ f(\hat{x}_{T}) - f(x^{*})] \le GB.
\end{equation}
\end{proof}
\begin{theorem}
\label{th6.3-bis}
Let $f\in {\cal F} _{\alpha, 2}(L)$ with $\alpha,L>0$, $\sigma>0$, and 
	  let Assumption \ref{ass1} hold. 
	  Consider the version of Algorithm \ref{algo} as in Theorem \ref{th6.3} where $\com$ is a convex compact subset  of $\mathbb{R}^d$ and {assume that $\max_{x\in\com}\|\nabla f(x)\|\le G $}.  If $h_{t}= \left(\frac{3d^2\sigma^2}{4L\alpha t+9L^2d^2}\right)^{1/4}$ and  $\eta_t=\frac{1}{\alpha t}$ then the optimization error of the estimator (\ref{est:s}) satisfies
	  \begin{equation}
	\label{eq0:th6.3-bis}
 \mathbb{E} [ f(\hat{x}_{T}) - f(x^{*})]
	\le 
	\min\left(GB, A_{8}
\frac{d}{\sqrt{\alpha T}} + A_{9}\frac{d^{2}}{\alpha T}\right),
	\end{equation}
	where $x^{*} = \argmin_{x \in \Theta}f(x)$. Here $A_{8}$ and $A_{9}$ are  positive constants that do not depend on $d,\alpha,T$, and $B$ is the Euclidean diameter of $\com$.
\end{theorem}
\begin{proof}
Arguing as in the proof of Theorem \ref{th6.3} 
but taking the sum over $\floor*{{T}/{2}}+1,\dots,T$ rather than over $1,\dots,T$ we obtain
\begin{align*}
\label{ex:y2}
 \sum_{t=\floor*{{T}/{2}}+1}^{T} \mathbb{E}\big[f(x_t) - f(x^{*})\big]  
&\le r_{\floor*{{T}/{2}}+1}\frac{\floor*{{T}/{2}}\alpha}{2}+
\sum_{t=\floor*{{T}/{2}}+1}^{T} \Big[ \Big( L +  \frac{9L^2d^2}{4\alpha t} \Big)h_t^2 
+ \frac{3 d^2 \sigma^2}{4h_t^2\alpha t}+
 \frac{9G^2d}{2\alpha t} \Big]
 \\
&\le 
 r_{\floor*{{T}/{2}}+1}\frac{\floor*{{T}/{2}}\alpha}{2}+\sum_{t=\floor*{{T}/{2}}+1}^{T}  \Big[
\sqrt{3}\frac{d\sigma\sqrt{L}}{\sqrt{\alpha t}} + \frac{3\sqrt{3}Ld^2\sigma}{2\alpha t}  + \frac{9G^2d}{2\alpha t} \Big]
\\
&\le 
 r_{\floor*{{T}/{2}}+1}\frac{\floor*{{T}/{2}}\alpha}{2}
 + 2\sqrt{3L}\sigma \frac{d}{\sqrt{\alpha }}\sqrt{T}
 +\frac{3d}{2\alpha }(
\sqrt{3} Ld\sigma + {3}G^2 )\sum_{t=\floor*{{T}/{2}}+1}^{T} \frac{1}{t}
\\
&\le 
 r_{\floor*{{T}/{2}}+1}\frac{\floor*{{T}/{2}}\alpha}{2}
 + 2\sqrt{3L}\sigma \frac{d}{\sqrt{\alpha }}\sqrt{T}
 +\frac{3d}{\alpha }(
\sqrt{3} Ld\sigma + {3}G^2 ),
\end{align*}
where we have used the inequality $\sum_{t=\floor*{{T}/{2}}+1}^{T} \frac{1}{t} \le \log(T/\floor*{{T}/{2}}) \le 2$ for all $T\ge 2$ (recall that we assume $T\ge 2$ throughout the paper). 
It follows from Theorem \ref{th:0-bis}, that
\begin{equation*}
\label{eq:yt}
	r_{\floor*{{T}/{2}}+1} < \frac{4G^{2}}{\alpha^{2}T}+\sqrt{2}A_{6}\frac{d}{\alpha^{\frac{3}{2}}T^{\frac{1}{2}}}+2A_{7}\frac{d^{2}}{\alpha^{2}T}.
	\end{equation*}
Combining the last two displays yields
\begin{align}
 \sum_{t=\floor*{{T}/{2}}+1}^{T}
 \mathbb{E}\big[f(x_t) - f(x^{*})\big]  
&\le \frac{G^{2}}{\alpha }+A_{6}\frac{d}{2\sqrt{2}\sqrt{\alpha}}\sqrt{T}+A_{7}\frac{d^{2}}{2\alpha}+
2\sqrt{3L}\sigma \frac{d}{\sqrt{\alpha }}\sqrt{T}+\frac{3d}{\alpha }(
\sqrt{3} Ld\sigma + {3}G^2 ) . \nonumber
\end{align}
From this inequality, using the fact that $f$ is a convex function, we obtain
\begin{align}
\mathbb{E} [f(\hat{x}_{T})-f(x^{*})] 
&\le A_{8}\frac{d}{\sqrt{\alpha T}}+A_{9}\frac{d^{2}}{\alpha T},
\nonumber
\end{align}
where $A_8= \frac{A_{6}}{\sqrt{2}} + 4\sqrt{3L}\sigma 
$ and $A_9=A_7 + 2(3\sqrt{3} L\sigma  + {(9 d +1)G^2}/{d^2} )$.
\end{proof}
\begin{theorem}
\label{th6.3-bisx}
	Let $f\in {\cal F} _{\alpha, 2}(L)$ with $\alpha,L>0$, and let Assumption \ref{ass1} hold. Consider the version of Algorithm \ref{algo} as in Theorem \ref{th6.3} where $\com$ is a convex compact subset  of $\mathbb{R}^d$, and $h_t = \left(\frac{3d^2\sigma^2}{4L\alpha t} \right)^{\frac{1}{4}} $, $\eta_t=\frac{1}{\alpha t}$. If $f$ is Lipschitz continuous with Lipschitz constant $G$ on the Euclidean $h_1$-neighborhood of $\com$, then for $\sigma>0$ we have the following bound for the cumulative regret:
	  \begin{equation}
	\label{eq0:th6.3-bisx}
	\forall x\in \com: \ \sum_{t=1}^T \mathbb{E} [ f(x_t) - f(x)]
	\le 
	\min\left(GBT, 2\sqrt{3L}\sigma
\frac{d}{\sqrt{\alpha}} \sqrt{T}+ \frac{C^*G^2}{2}\frac{d}{\alpha} (1+\log T)\right),
	\end{equation}
	where $B$ is the Euclidean diameter of $\com$. 
	
	If $\sigma=0$, then the cumulative regret for any $h_t$ chosen small enough and $\eta_t=\frac{1}{\alpha t}$ satisfies
		\begin{equation*}
	\label{eq:b}
	\forall x\in \com: \ \sum_{t=1}^T \mathbb{E} [ f(x_t) - f(x)]\le 
		\min\left(GBT, C^* G^2  \,\frac{ d}{\alpha} (1+\log T) \right)
	\end{equation*}
\end{theorem}
\begin{proof}
The  argument is analogous to the proof of Theorem \ref{th6.3}. The difference is only in the bound on $
\mathbb{E} [\|{\hat g}_t \|^2 ]$. To evaluate this term, we now use Lemma \ref{lem:2} (noticing that when $K(\cdot)\equiv 1$ this lemma is satisfied with $\kappa=1$). This yields
\begin{align}\label{eq1:th6.3-bis}
\forall x\in \Theta: \qquad \mathbb{E} \sum_{t=1}^{T} \big(f(x_t) - f(x)\big)  
&\le \sum_{t=1}^{T} \Big[ Lh_t^2 + \frac{1}{2\alpha t}\Big( C^*G^2d +  \frac{3 d^2 \sigma^2}{2h_t^2}\Big)
 \Big]. 
\end{align}
The chosen value $h_t = \left(\frac{3d^2\sigma^2}{4L\alpha t} \right)^{\frac{1}{4}} $ minimizes the r.h.s. and together with \eqref{ef8:th6.3} yields 
\eqref{eq0:th6.3-bisx}. The remaining part of the proof follows the same lines as in Theorem \ref{thm:1}.
\end{proof}
\end{appendices}
\end{document}